\documentclass[english]{article}
\usepackage[T1]{fontenc}
\usepackage[latin9]{inputenc}
\usepackage{color}
\usepackage{float}
\usepackage{amsthm}
\usepackage{amsmath}
\usepackage{amssymb}
\usepackage{graphicx}

\makeatletter

\providecommand{\tabularnewline}{\\}
\newcommand{\lyxdot}{.}

\theoremstyle{plain}
\newtheorem{thm}{\protect\theoremname}
\theoremstyle{definition}
\newtheorem{defn}[thm]{\protect\definitionname}
\theoremstyle{plain}
\newtheorem{lem}[thm]{\protect\lemmaname}
\ifx\proof\undefined
\newenvironment{proof}[1][\protect\proofname]{\par
\normalfont\topsep6\p@\@plus6\p@\relax
\trivlist
\itemindent\parindent
\item[\hskip\labelsep
\scshape
#1]\ignorespaces
}{%
\endtrivlist\@endpefalse
}
\providecommand{\proofname}{Proof}
\fi

\usepackage{nips11submit_e}\usepackage{times}

\nipsfinalcopy 

\usepackage{enumitem,amsthm}
\setenumerate{leftmargin=.25in} 

\usepackage{psfrag}

\usepackage{endnotes}
\usepackage{hyperref}

\makeatother

\usepackage{babel}
\providecommand{\definitionname}{Definition}
\providecommand{\lemmaname}{Lemma}
\providecommand{\theoremname}{Theorem}

\begin{document}

\title{Projecting Ising Model Parameters for Fast Mixing}

\author{Justin Domke\\
NICTA, The Australian National University\\
\texttt{justin.domke@nicta.com.au} \and \textbf{Xianghang Liu}\\
NICTA, The University of New South Wales\\
\texttt{xianghang.liu@nicta.com.au}}
\maketitle
\begin{abstract}
Inference in general Ising models is difficult, due to high treewidth
making tree-based algorithms intractable. Moreover, when interactions
are strong, Gibbs sampling may take exponential time to converge to
the stationary distribution. We present an algorithm to project Ising
model parameters onto a parameter set that is guaranteed to be fast
mixing, under several divergences. We find that Gibbs sampling using
the projected parameters is more accurate than with the original parameters
when interaction strengths are strong and when limited time is available
for sampling.
\end{abstract}

\section{Introduction}

\textcolor{cyan}{}High-treewidth graphical models typically yield
distributions where exact inference is intractable. To cope with this,
one often makes an approximation based on a tractable model. For example,
given some intractable distribution $q$, mean-field inference \cite{AMeanFieldTheoryLearningAlgorithm}
attempts to minimize $KL(p||q)$ over $p\in\text{TRACT}$, where $\text{TRACT}$
is the set of fully-factorized distributions. Similarly, structured
mean-field minimizes $ $the KL-divergence, but allows $\text{TRACT}$
to be the set of distributions that obey some tree \cite{ExploitingTractableSubstructures}
or a non-overlapping clustered \cite{AGeneralizedMeanField} structure.
In different ways, loopy belief propagation \cite{Yedidia05constructingfree}
and tree-reweighted belief propagation \cite{Wainwright_anewclass}
also make use of tree-based approximations, while Globerson and Jaakkola
\cite{ApproximateInferenceUsingPlanarGraph} provide an approximate
inference method based on exact inference in planar graphs with zero
field.\textcolor{cyan}{}

In this paper, we explore an alternative notion of a ``tractable''
model. These are ``fast mixing'' models, or distributions that,
while they may be high-treewidth, have parameter-space conditions
guaranteeing that Gibbs sampling will quickly converge to the stationary
distribution. While the precise form of the parameter space conditions
is slightly technical (Sections \ref{sec:Background}-\ref{sec:Mixing-Time-Bounds}),
informally, it is simply that interaction strengths between neighboring
variables are not too strong.

In the context of the Ising model, we attempt to use these models
in the most basic way possible-- by taking an arbitrary (slow-mixing)
set of parameters, projecting onto the fast-mixing set, using four
different divergences. First, we show how to project in the Euclidean
norm, by iteratively thresholding a singular value decomposition (Theorem
\ref{thm:projection_theorem-main}). Secondly, we experiment with
projecting using the ``zero-avoiding'' divergence $KL(q||p)$. Since
this requires taking (intractable) expectations with respect to $q$,
it is of only theoretical interest. \textcolor{cyan}{} Third, we
suggest a novel ``piecewise'' approximation of the KL divergence,
where one drops edges from both $q$ and $p$ until a low-treewidth
graph remains where the exact KL divergence can be calculated. Experimentally,
this does not perform as well as the true KL-divergence, but is easy
to evaluate. Fourth, we consider the ``zero forcing'' divergence
$KL(q||p)$. Since this requires expectations with respect to $p$,
which is constrained to be fast-mixing, it can be approximated by
Gibbs sampling, and the divergence can be minimized through stochastic
approximation. This can be seen as a generalization of mean-field
where the set of approximating distributions is expanded from fully-factorized
to fast-mixing.

\section{Background\label{sec:Background}}

The literature on mixing times in Markov chains is extensive, including
a recent textbook \cite{MarkovChainsAndMixingTimes}. The presentation
in the rest of this section is based on that of Dyer et al. \cite{MatrixNormsAndRapidMixingForSpinSystems}.

Given a distribution $p(x)$, one will often wish to draw samples
from it. While in certain cases (e.g. the Normal distribution) one
can obtain exact samples, for Markov random fields (MRFs), one must
generally resort to iterative Markov chain Monte Carlo (MCMC) methods
that obtain a sample asymptotically. In this paper, we consider the
classic Gibbs sampling method \cite{StochasticRelaxationGibbsDistributions},
where one starts with some configuration $x$, and repeatedly picks
a node $i$, and samples $x_{i}$ from $p(x_{i}|x_{-i})$. Under mild
conditions, this can be shown to sample from a distribution that converges
to $p$ as $t\rightarrow\infty$.

It is common to use more sophisticated methods such as block Gibbs
sampling, the Swendsen-Wang algorithm \cite{SwendsenWang}, or tree
sampling \cite{FromFieldsToTrees}. In principle, each algorithm could
have unique parameter-space conditions under which it is fast mixing.
Here, we focus on the univariate case for simplicity and because fast
mixing of univariate Gibbs is sufficient for fast mixing of some other
methods \cite{ExtraUpdates}.
\begin{defn}
Given two finite distributions $p$ and $q$, the \textbf{total variation
distance} $||\cdot||_{TV}$ is
\[
||p(X)-q(X)||_{TV}=\frac{1}{2}\sum_{x}|p(X=x)-q(X=x)|.
\]

\end{defn}
We need a property of a distribution that can guarantee fast mixing.
The \textbf{dependency} $R_{ij}$ of $x_{i}$ on $x_{j}$ is defined
by considering two configurations $x$ and $x'$, and measuring how
much the conditional distribution of $x_{i}$ can vary when $x_{k}=x'_{k}$
for all $k\not=j$.
\begin{defn}
Given a distribution $p$, the dependency matrix $R$ is defined by 

\[
R_{ij}=\max_{x,x':x_{-j}=x_{-j}'}||p(X_{i}|x_{-i})-p(X_{i}|x_{-i}')||_{TV}.
\]

\end{defn}
Given some threshold $\epsilon$, the \textbf{mixing time} is the
number of iterations needed to guarantee that the total variation
distance of the Gibbs chain to the stationary distribution is less
than $\epsilon$.
\begin{defn}
Suppose that $\{X^{t}\}$ denotes the sequence of random variables
corresponding to running Gibbs sampling on some distribution $p$.
The mixing time $\tau(\epsilon)$ is the minimum time $t$ such that
the total variation distance between $X^{t}$ and the stationary distribution
is at most $\epsilon$. That is,
\begin{align*}
\tau(\epsilon)= & \min\{t:d(t)<\epsilon\},\\
d(t)= & \max_{x}||\mathbb{P}(X^{t}|X^{0}=x)-p(X)||_{TV}.
\end{align*}

\end{defn}
Unfortunately, the mixing time can be extremely long, which makes
the use of Gibbs sampling delicate in practice. For example, for the
two-dimensional Ising model with zero field and uniform interactions,
it is known that mixing time is polynomial (in the size of the grid)
when the interaction strengths are below a threshold $\beta_{c}$,
and exponential for stronger interactions \cite{CriticalIsingMixing}.
For more general distributions, such tight bounds are not generally
known, but one can still derive sufficient conditions for fast mixing.
The main result we will use is the following \cite{ASimpleConditionImplyingRapidMixing}.
\begin{thm}
\label{thm:mixing-time-theorem}Consider the dependency matrix $R$
corresponding to some distribution $p(X_{1},...,X_{n})$. For Gibbs
sampling with random updates, if $||R||_{2}<1,$ the mixing time is
bounded by

\[
\tau(\epsilon)\leq\frac{n}{1-||R||_{2}}\ln\left(\frac{n}{\epsilon}\right).
\]

\end{thm}
Roughly speaking, if the spectral norm (maximum singular value) of
$R$ is less than one, rapid mixing will occur. A similar result holds
in the case of systematic scan updates \cite{MatrixNormsAndRapidMixingForSpinSystems,ASimpleConditionImplyingRapidMixing}.

Some of the classic ways of establishing fast mixing can be seen as
special cases of this. For example, the Dobrushin criterion is that
$||R||_{1}<1$, which can be easier to verify in many cases, since
$||R||_{1}=\max_{j}\sum_{i}|R_{ij}|$ does not require the computation
of singular values. However, for symmetric matrices, it can be shown
that $||R||_{2}\leq||R||_{1}$, meaning the above result is tighter.

\section{Mixing Time Bounds\label{sec:Mixing-Time-Bounds}}

For variables $x_{i}\in\{-1,+1\},$ an Ising model is of the form
\[
p(x)=\exp\left(\sum_{i,j}\beta_{ij}x_{i}x_{j}+\sum_{i}\alpha_{i}x_{i}-A(\beta,\alpha)\right),
\]
where $\beta_{ij}$ is the interaction strength between variables
$i$ and $j$, $\alpha_{i}$ is the ``field'' for variable $i$,
and $A$ ensures normalization. This can be seen as a member of the
exponential family $p(x)=\exp\left(\theta\cdot f(x)-A(\theta)\right),$
where $f(x)=\{x_{i}x_{j}\forall(i,j)\}\cup\{x_{i}\forall i\}$ and
$\theta$ contains both $\beta$ and $\alpha$.
\begin{lem}
\label{lem:ising-dependency}For an Ising model, the dependency matrix
is bounded by
\[
R_{ij}\leq\tanh|\beta_{ij}|\leq|\beta_{ij}|
\]

\end{lem}
Hayes \cite{ASimpleConditionImplyingRapidMixing} proves this for
the case of constant $\beta$ and zero-field, but simple modifications
to the proof can give this result.

Thus, to summarize, an Ising model can be guaranteed to be fast mixing
if the spectral norm of the absolute value of interactions terms is
less than one.

\section{Projection\label{sec:Projection}}

In this section, we imagine that we have some set of parameters $\theta$,
not necessarily fast mixing, and would like to obtain another set
of parameters $\psi$ which are as close as possible to $\theta$,
but guaranteed to be fast mixing. This section derives a projection
in the Euclidean norm, while Section \ref{sec:Divergences} will build
on this to consider other divergence measures.

We will use the following standard result that states that given a
matrix $A$, the closest matrix with a maximum spectral norm can be
obtained by thresholding the singular values.\textcolor{cyan}{}
\begin{thm}
\label{thm:dense-projection}If $A$ has a singular value decomposition
$A=USV^{T}$, and $||\cdot||_{F}$ denotes the Frobenius norm, then
$B=\underset{B:||B||_{2}\leq c}{\arg\text{}\min}||A-B||_{F}$ can
be obtained as $B=US'V^{T},$ where $S_{ii}^{'}=\min(S_{ii},c^{2}).$
\end{thm}
We denote this projection by $B=\Pi_{c}[A]$. This is close to providing
an algorithm for obtaining the closest set of Ising model parameters
that obey a given spectral norm constraint. However, there are two
issues. First, in general, even if $A$ is sparse, the projected matrix
$B$ will be dense, meaning that projecting will destroy a sparse
graph structure. Second, this result constrains the spectral norm
of $B$ itself, rather than $R=|B|$, which is what needs to be controlled.
The theorem below provides a dual method that fixed these issues.

Here, we take some matrix $Z$ that corresponds to the graph structure,
by setting $Z_{ij}=0$ if $(i,j)$ is an edge, and $Z_{ij}=1$ otherwise.
Then, enforcing that $B$ obeys the graph structure is equivalent
to enforcing that $Z_{ij}B_{ij}=0$ for all $(i,j)$. Thus, finding
the closest set of parameters $B$ is equivalent to solving

\vspace{-20pt}
\begin{eqnarray}
\min_{B,D} &  & ||A-B||_{F}\text{ subject to }||D||_{2}\leq c,\,\, Z_{ij}D_{ij}=0,\,\, D=|B|.\label{eq:sparse-projection-minimization-singleline}
\end{eqnarray}

We find it convenient to solve this minimization by performing some
manipulations, and deriving a dual. The proof of this theorem is provided
in the appendix. To accomplish the maximization of $g$ over $M$
and $ $$\Lambda$, we use LBFGS-B \cite{ALimitedMemoryAlgorithmFroBoundConstrainedOptimization},
with bound constraints used to enforce that $M\geq0$.

The following theorem uses the ``triple dot product'' notation of
$A\cdot B\cdot C=\sum_{ij}A_{ij}B_{ij}C_{ij}$.
\begin{thm}
\label{thm:projection_theorem-main}Define $R=|A|$. The minimization
in Eq. \ref{eq:sparse-projection-minimization-singleline} is equivalent
to the problem of $\max_{M\geq0,\Lambda}g(\Lambda,M)$, where the
objective and gradient of $g$ are, for $D(\Lambda,M)=\Pi_{c}[R+M-\Lambda\odot Z],$

\begin{align}
g(\Lambda,M) & =\frac{1}{2}||D(\Lambda,M)-R||_{F}^{2}+\Lambda\cdot Z\cdot D(\Lambda,M) - M \cdot D(\Lambda,M) \label{eq:sparse-projection-dual-obj-1-1}\\
\frac{dg}{d\Lambda} & =Z\odot D(\Lambda,M)\label{eq:sparse-projection-dual-grad-1-1}\\
\frac{dg}{dM} & = -D(\Lambda,M).\label{eq:sparse-projection-dual-grad-2-1}
\end{align}

\end{thm}

\section{Divergences\label{sec:Divergences}}

Again, we would like to find a parameter vector $\psi$ that is close
to a given vector $\theta$, but is guaranteed to be fast mixing,
but with several notions of ``closeness'' that vary in terms of
accuracy and computational convenience. Formally, if $\Psi$ is the
set of parameters that we can guarantee to be fast mixing, and $D(\theta,\psi)$
is a divergence between $\theta$ and $\psi$, then we would like
to solve
\begin{equation}
\arg\min_{\psi\in\Psi}D(\theta,\psi).\label{eq:general_projection}
\end{equation}

As we will see, in selecting $D$ there appears to be something of
a trade-off between the quality of the approximation, and the ease
of computing the projection in Eq. \ref{eq:general_projection}.

In this section, we work with the generic exponential family representation
\vspace{-10pt}

\[
p(x;\theta)=\exp(\theta\cdot f(x)-A(\theta)).
\]
We use $\mu$ to denote the mean value of $f$. By a standard result,
this is equal to the gradient of $A$, i.e.

\[
\mu(\theta)=\sum_{x}p(x;\theta)f(x)=\nabla A(\theta).
\]
\vspace{-15pt}

\subsection{Euclidean Distance}

The simplest divergence is simply the $l_{2}$ distance between the
parameter vectors, $D(\theta,\psi)=||\theta-\psi||_{2}$. For the
Ising model, Theorem \ref{thm:projection_theorem-main} provides a
method to compute the projection $\arg\min_{\psi\in\Psi}||\theta-\psi||_{2}.$
While simple, this has no obvious probabilistic interpretation, and
other divergences perform better in the experiments below.

However, it also forms the basis of our projected gradient descent
strategy for computing the projection in Eq. \ref{eq:general_projection}
under more general divergences $D$. Specifically, we will do this
by iterating
\begin{enumerate}
\item $\psi'\leftarrow\psi-\lambda\frac{d}{d\psi}D(\theta,\psi)$
\item $\psi\leftarrow\arg\min_{\psi\in\Psi}||\psi'-\psi||_{2}$
\end{enumerate}
for some step-size $\lambda$. In some cases, $dD/d\psi$ can be calculated
exactly, and this is simply projected gradient descent. In other cases,
one needs to estimate $dD/d\psi$ by sampling from $\psi$. As discussed
below, we do this by maintaining a ``pool'' of samples. In each
iteration, a few Markov chain steps are applied with the current parameters,
and then the gradient is estimated using them. Since the gradients
estimated at each time-step are dependent, this can be seen as an
instance of Ergodic Mirror Descent \cite{ErgodicMirrorDescent}. This
guarantees convergence if the number of Markov chain steps, and the
step-size $\lambda$ are both functions of the total number of optimization
iterations.

\subsection{KL-Divergence}

Perhaps the most natural divergence to use would be the ``inclusive''
KL-divergence
\begin{equation}
D(\theta,\psi)=KL(\theta||\psi)=\sum_{x}p(x;\theta)\log\frac{p(x;\theta)}{p(x;\psi)}.\label{eq:zero-avoiding-KL}
\end{equation}

This has the ``zero-avoiding'' property \cite{DivergenceMeasuresAndMessagePassing}
that $\psi$ will tend to assign some probability to all configurations
that $\theta$ assigns nonzero probability to. It is easy to show
that the derivative is
\begin{equation}
\frac{dD(\theta,\psi)}{d\psi}=\mu(\psi)-\mu(\theta),\label{eq:KL-divergence-gradient}
\end{equation}
where $\mu_{\theta}=\mathbb{E}_{\theta}[f(X)].$ Unfortunately, this
requires inference with respect to both the parameter vectors $\theta$
and $\psi$. Since $\psi$ will be enforced to be fast-mixing during
optimization, one could approximate $\mu(\psi)$ by sampling. However,
$\theta$ is presumed to be slow-mixing, making $\mu(\theta)$ difficult
to compute. Thus, this divergence is only practical on low-treewidth
``toy'' graphs.

\subsection{Piecewise KL-Divergences}

Inspired by the piecewise likelihood \cite{PiecewiseTraining} and
likelihood approximations based on mixtures of trees \cite{SpanningTreeApproximations},
we seek tractable approximations of the KL-divergence based on tractable
subgraphs. Our motivation is the the following: if $\theta$ and $\psi$
define the same distribution, then if a certain set of edges are removed
from both, they should continue to define the same distribution%
\footnote{Technically, here, we assume that the exponential family is minimal.
However, in the case of an over-complete exponential family, enforcing
this will simply ensure that $\theta$ and $\psi$ use the same reparameterization.%
}. Thus, given some graph $T$, we define the ``projection'' $\theta(T)$
onto the tree such by setting all edge parameters to zero if not part
of $T$. Then, given a set of graphs $T$, the piecewise KL-divergence
is
\[
D(\theta,\psi)=\max_{T}KL(\theta(T)||\psi(T)).
\]
Computing the derivative of this divergence is not hard-- one simply
computes the KL-divergence for each graph, and uses the gradient as
in Eq. \ref{eq:KL-divergence-gradient} for the maximizing graph.

There is some flexibility of selecting the graphs $T$. In the simplest
case, one could simply select a set of trees (assuring that each edge
is covered by one tree), which makes it easy to compute the KL-divergence
on each tree using the sum-product algorithm. We will also experiment
with selecting low-treewidth graphs, where exact inference can take
place using the junction tree algorithm.

\subsection{Reversed KL-Divergence}

We also consider the ``zero-forcing'' KL-divergence
\[
D(\theta,\psi)=KL(\psi||\theta)=\sum_{x}p(x;\psi)\log\frac{p(x;\psi)}{p(x;\theta)}.
\]

\begin{thm}
The divergence $D(\theta,\psi)=KL(\psi||\theta)$ has the gradient
\[
\frac{d}{d\psi}D(\theta,\psi)=\sum_{x}p(x;\psi)(\psi-\theta)\cdot f(x)\left(f(x)-\mu(\psi)\right).
\]
\vspace{-10pt}

\end{thm}
Arguably, using this divergence is inferior to the ``zero-avoiding''
KL-divergence. For example, since the parameters $\psi$ may fail
to put significant probability at configurations where $\theta$ does,
using importance sampling to reweight samples from $\psi$ to estimate
expectations with respect to $\theta$ could have high variance Further,
it can be non-convex with respect to $\psi$. Nevertheless, it often
work well in practice. Minimizing this divergence under the constraint
that the dependency matrix $R$ corresponding to $\psi$ have a limited
spectral norm is closely related to naive mean-field, which can be
seen as a degenerate case where one constrains $R$ to have zero norm.

This is easier to work with than the ``zero-avoiding'' KL-divergence
in Eq. \ref{eq:zero-avoiding-KL} since it involves taking expectations
with respect to $\psi$, rather than $\theta$: since $\psi$ is enforced
to be fast-mixing, these expectations can be approximated by sampling.
Specifically, suppose that one has generated a set of samples $x^{1},...,x^{K}$
using the current parameters $\psi$. Then, one can first approximate
the marginals by $\hat{\mu}=\frac{1}{K}\sum_{k=1}^{K}f(x^{k}),$ and
then approximate the gradient by
\begin{equation}
\hat{g}=\frac{1}{K}\sum_{k=1}^{K}\left((\psi-\theta)\cdot f(x^{k})\right)\left(f(x^{k})-\hat{\mu}\right).\label{eq:SGD_gradient_estimate}
\end{equation}
\vspace{-10pt}

It is a standard result that if two estimators are unbiased and independent,
the product of the two estimators will also be unbiased. Thus, if
one used separate sets of perfect samples to estimate $\hat{\mu}$
and $\hat{g}$, then $\hat{g}$ would be an unbiased estimator of
$dD/d\psi$. In practice, of course, we generate the samples by Gibbs
sampling, so they are not quite perfect. We find in practice that
using the same set of samples twice makes makes little difference,
and do so in the experiments.
\begin{figure}[H]
\begin{centering}
\hfill{}\includegraphics[width=0.445\textwidth]{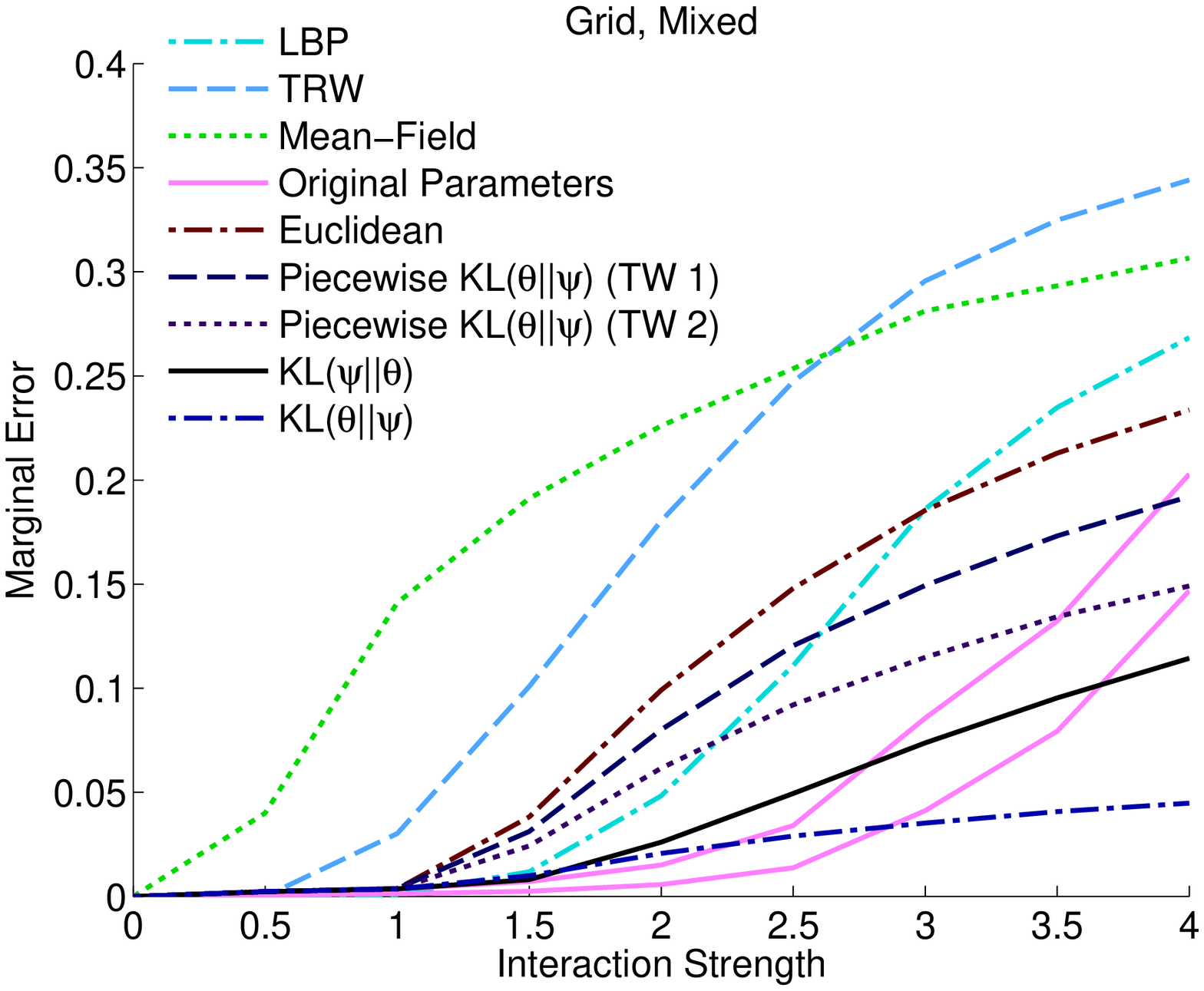}\hfill{}\includegraphics[width=0.445\textwidth]{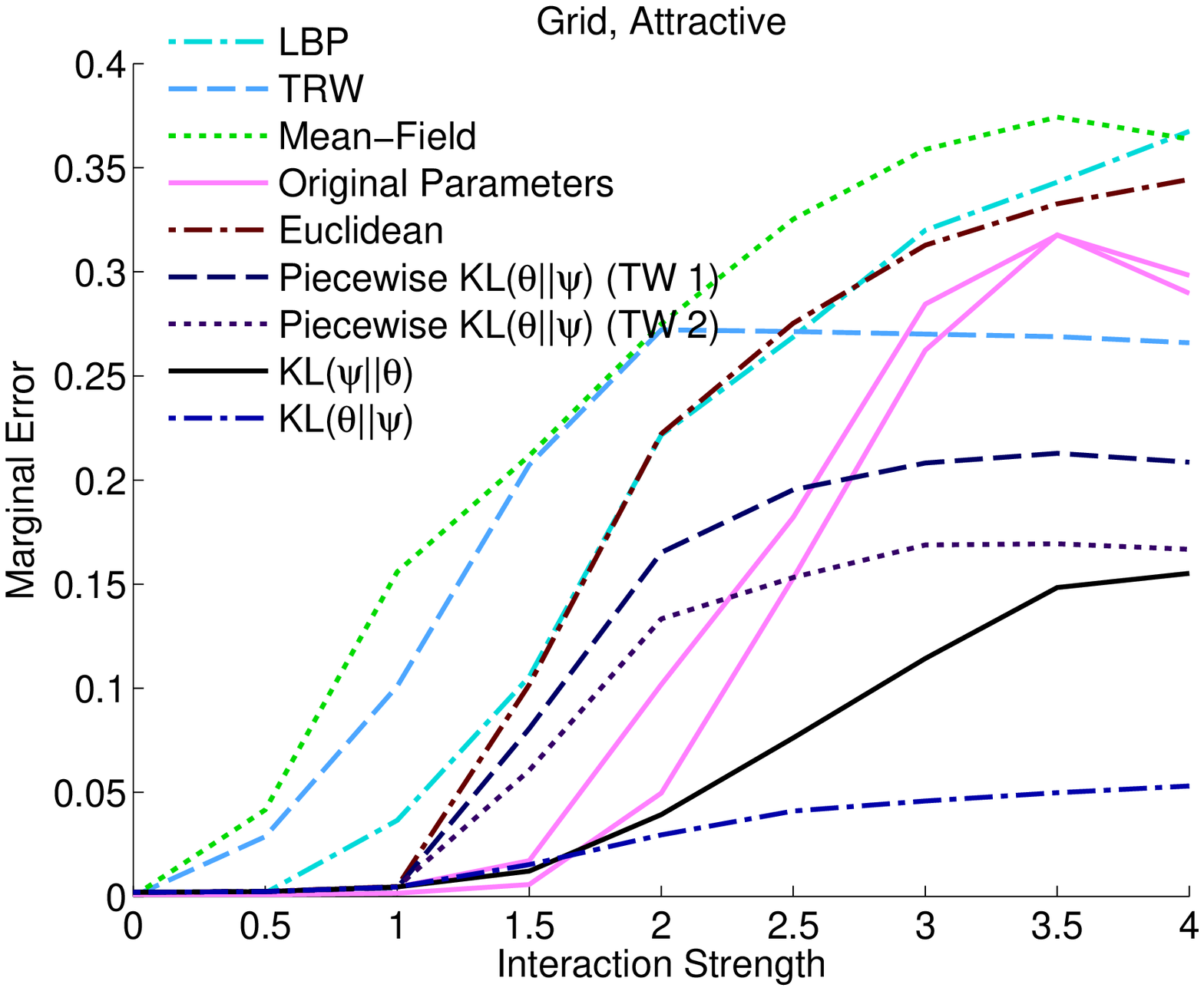}\hfill{}
\par\end{centering}

\begin{centering}
\hfill{}\includegraphics[width=0.445\textwidth]{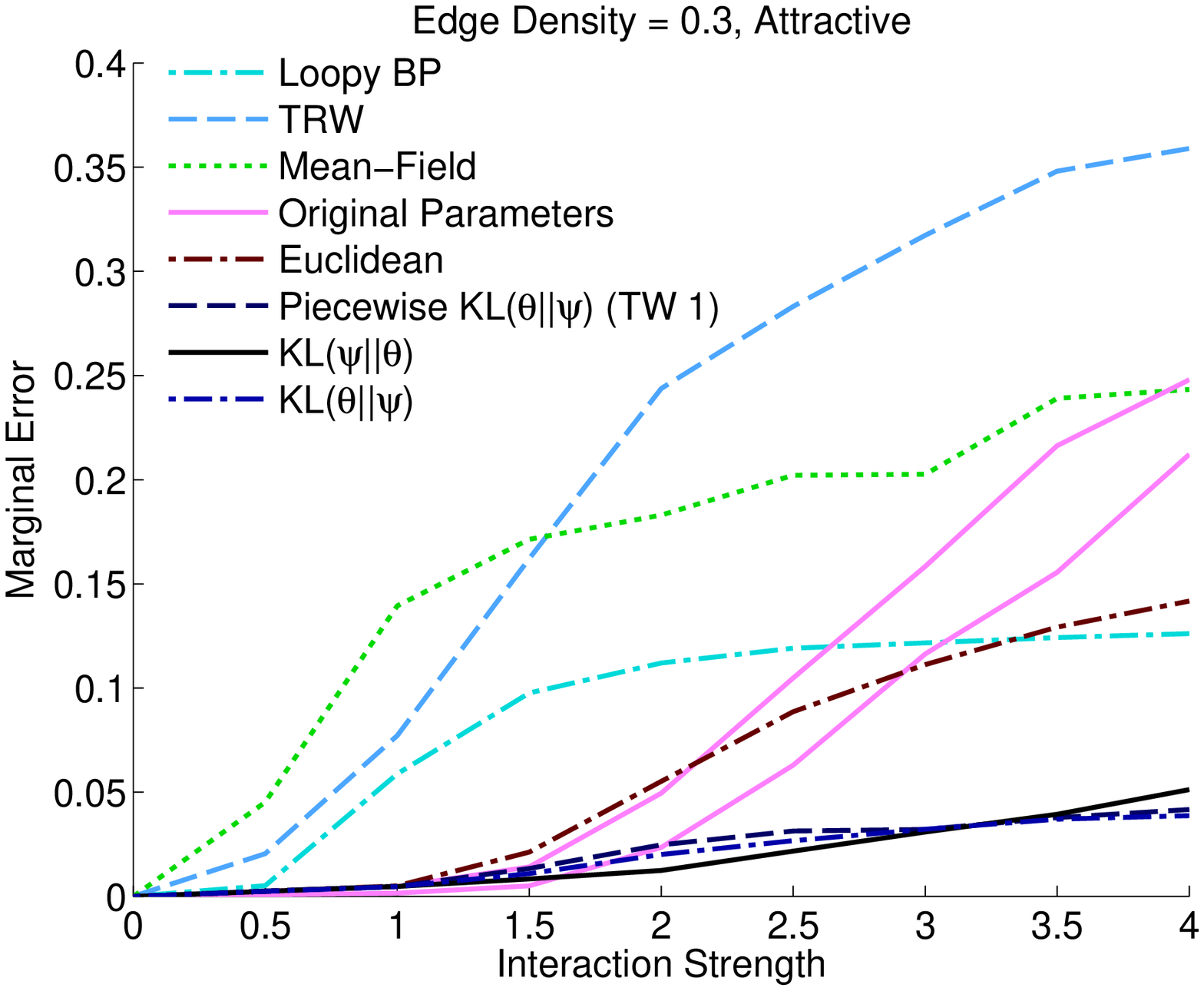}\hfill{}\includegraphics[width=0.445\textwidth]{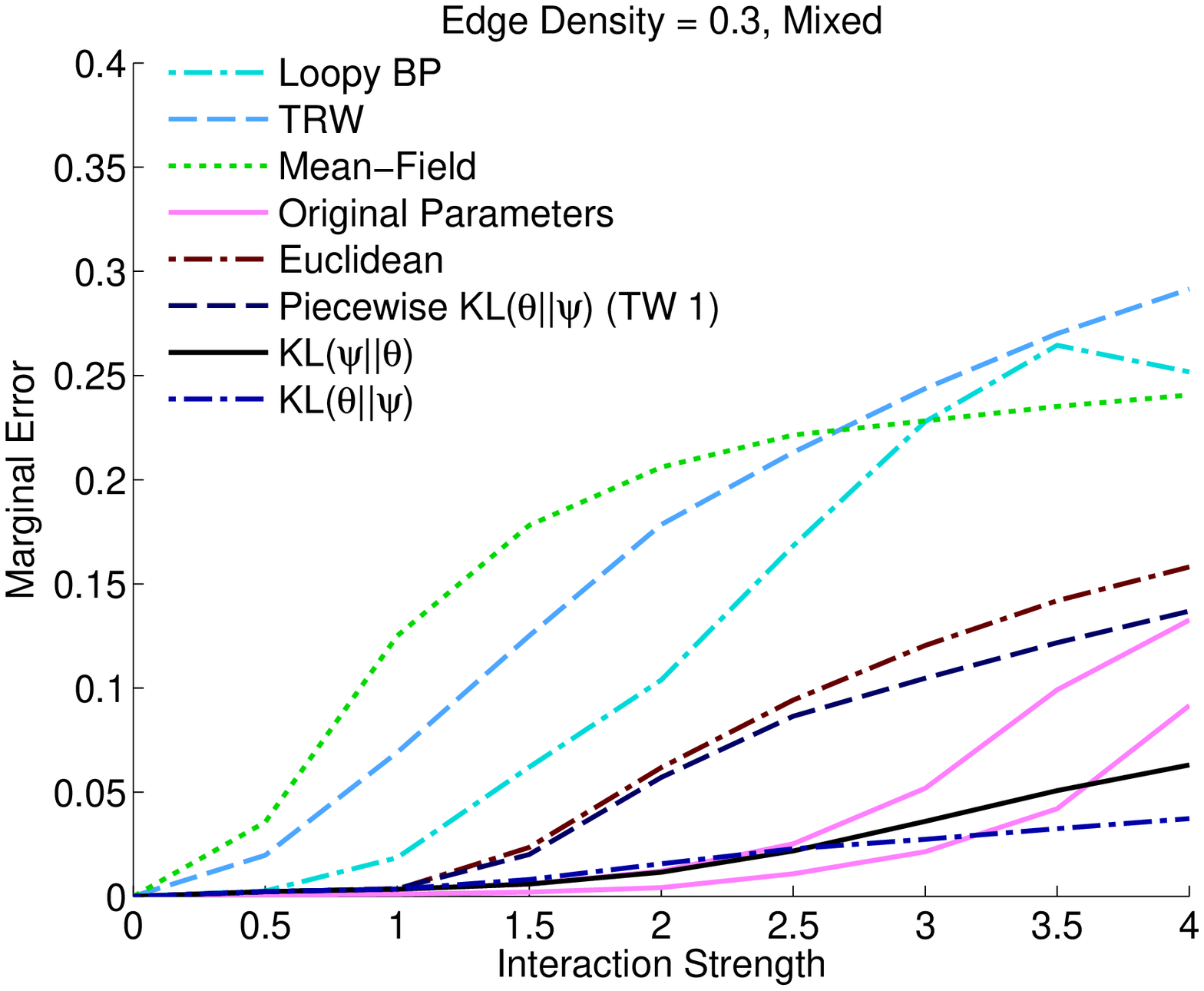}\hfill{}\vspace{-10pt}

\par\end{centering}

\caption{The mean error of estimated univariate marginals on 8x8 grids (top
row) and low-density random graphs (bottom row), comparing 30k iterations
of Gibbs sampling after projection to variational methods. To approximate
the computational effort of projection (Table \ref{tab:Running-Times}),
sampling on the original parameters with 250k iterations is also included
as a lower curve. (Full results in appendix.)}
\end{figure}

\section{Experiments}

Our experimental evaluation follows that of Hazan and Shashua \cite{ConvergentMessagePassingAlgorithms}
in evaluating the accuracy of the methods using the Ising model in
various configurations. In the experiments, we approximate randomly
generated Ising models with rapid-mixing distributions using the projection
algorithms described previously. Then, the marginals of rapid-mixing
approximate distribution are compared against those of the target
distributions by running a Gibbs chain on each. We calculate the mean
absolute distance of the marginals as the accuracy measure, with the
marginals computed via the exact junction-tree algorithm.  

We evaluate projecting under the Euclidean distance (Section 5.1),
the piecewise divergence (Section 5.3)\textcolor{black}{, and the
zero-forcing KL-divergence $KL(\psi||\theta)$} (Section 5.4). On
small graphs, it is possible to minimize the zero-avoiding KL-divergence
$KL(\theta||\psi)$ by computing marginals using the junction-tree
algorithm. However, as minimizing this KL-divergence leads to exact
marginal estimates, it doesn't provide a useful measure of marginal
accuracy. Our methods are compared with four other inference algorithms,
namely loopy belief-propagation (LBP), Tree-reweighted belief-propagation
(TRW), Naive mean-field (MF), and Gibbs sampling on the original parameters.

LBP, MF and TRW are among the most widely applied variational methods
for approximate inference. The MF algorithm uses a fully factorized
distribution as the tractable family, and can be viewed as an extreme
case of\textcolor{black}{{} minimizing the zero forcing KL-divergence
$KL(\psi||\theta)$} under the constraint of zero spectral norm. The
tractable family that it uses guarantees ``instant'' mixing but
is much more restrictive. Theoretically, Gibbs sampling on the original
parameters will produce highly accurate marginals if run long enough.
However, this can take exponentially long and convergence is generally
hard to diagnose \cite{Cowles96markovchain}\textcolor{cyan}{.} In
contrast, Gibbs sampling on the rapid-mixing approximation is guaranteed
to converge rapidly but will result in less accurate marginals asymptotically.
Thus, we also include time-accuracy comparisons between these two
strategies in the experiments.

\subsection{Configurations}

Two types of graph topologies are used: two-dimensional $8\times8$
grids and random graphs with $10$ nodes. Each edge is independently
present with probability $p_{e}\in\{0.3,0.5,0.7\}$. Node parameters
$\theta_{i}$ are uniformly drawn from unif$(-d_{n},d_{n})$ and we
fix the field strength to $d_{n}=1.0$. Edge parameters $\theta_{ij}$
are uniformly drawn from unif$(-d_{e},d_{e})$ or unif$(0,d_{e})$
to obtain mixed or attractive interactions respectively. We generate
graphs with different interaction strength $d_{e}=0,0.5,\dots,4$.
All results are averaged over 50 random trials.

To calculate piecewise divergences, it remains to specify the set
of subgraphs $T$. It can be any tractable subgraph of the original
distribution. For the grids, one straightforward choice is to use
the horizontal and the vertical chains as subgraphs. We also test
with chains of treewidth 2. For random graphs, we use\textcolor{black}{{}
the sets of random spanning trees which can cover every edge of the
original graphs as the set of subgraphs}.

\textcolor{black}{A stochastic gradient descent algorithm is applied
to minimize the zero forcing KL-divergence $KL(\psi||\theta)$. In
this algorithm, a ``pool'' of samples is repeatedly used to estimate
gradients as in Eq. \ref{eq:SGD_gradient_estimate}. After each parameter
update, each sample is updated by a single Gibbs step, consisting
of one pass over all variables. The performance of this algorithm
can be affected by several parameters, including the gradient search
step size, the size of the sample pool, the number of Gibbs updates,
and the number of total iterations. (This algorithm can be seen as
an instance of Ergodic Mirror Descent \cite{ErgodicMirrorDescent}.)
Without intensive tuning of these parameters, we choose a constant
step size of $0.1$, sample pool size of 500 and 60 total iterations,
which performed reasonably well in practice. }

For each original or approximate distribution, a single chain of Gibbs
sampling is run on the final parameters, and marginals are estimated
from the samples drawn.\textcolor{black}{{} Each Gibbs iteration is
one pass of systematical scan over the variables in fixed order. Note
that this does not take into account the computational effort deployed
during projection, which ranges from 30,000 total Gibbs iterations
with repeated Euclidean projection $(KL(\psi||\theta)$) to none at
all (Original parameters). It has been our experience that more aggressive
parameters can lead to this procedure being more accurate than Gibbs
in a comparison of total computational effort, but such a scheduling
tends to also reduce the accuracy of the final parameters, making
results more difficult to interpret.}

In Section 3.2, we show that for Ising models, a sufficient condition
for rapid-mixing is the spectral norm of pairwise weight matrix is
less than 1.0. However, we find in practice using a spectral norm
bound of $2.5$ instead of $1.0$ can still preserve the rapid-mixing
property and gives better approximation to the original distributions.
(See Section \ref{sec:Discussion} for a discussion.)

\begin{figure}[t]
\begin{centering}
\hfill{}\includegraphics[width=0.445\textwidth]{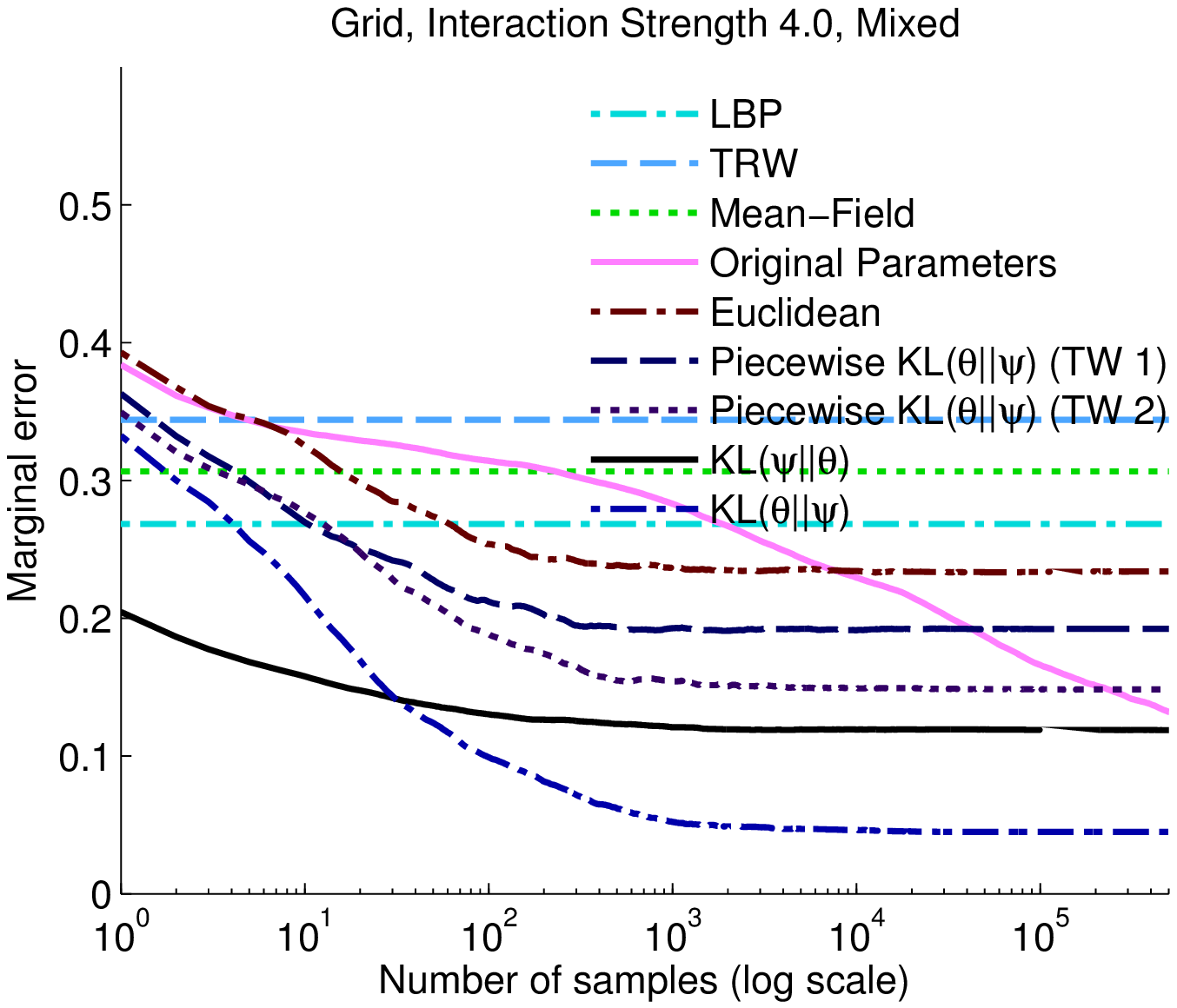}\hfill{}\includegraphics[width=0.445\textwidth]{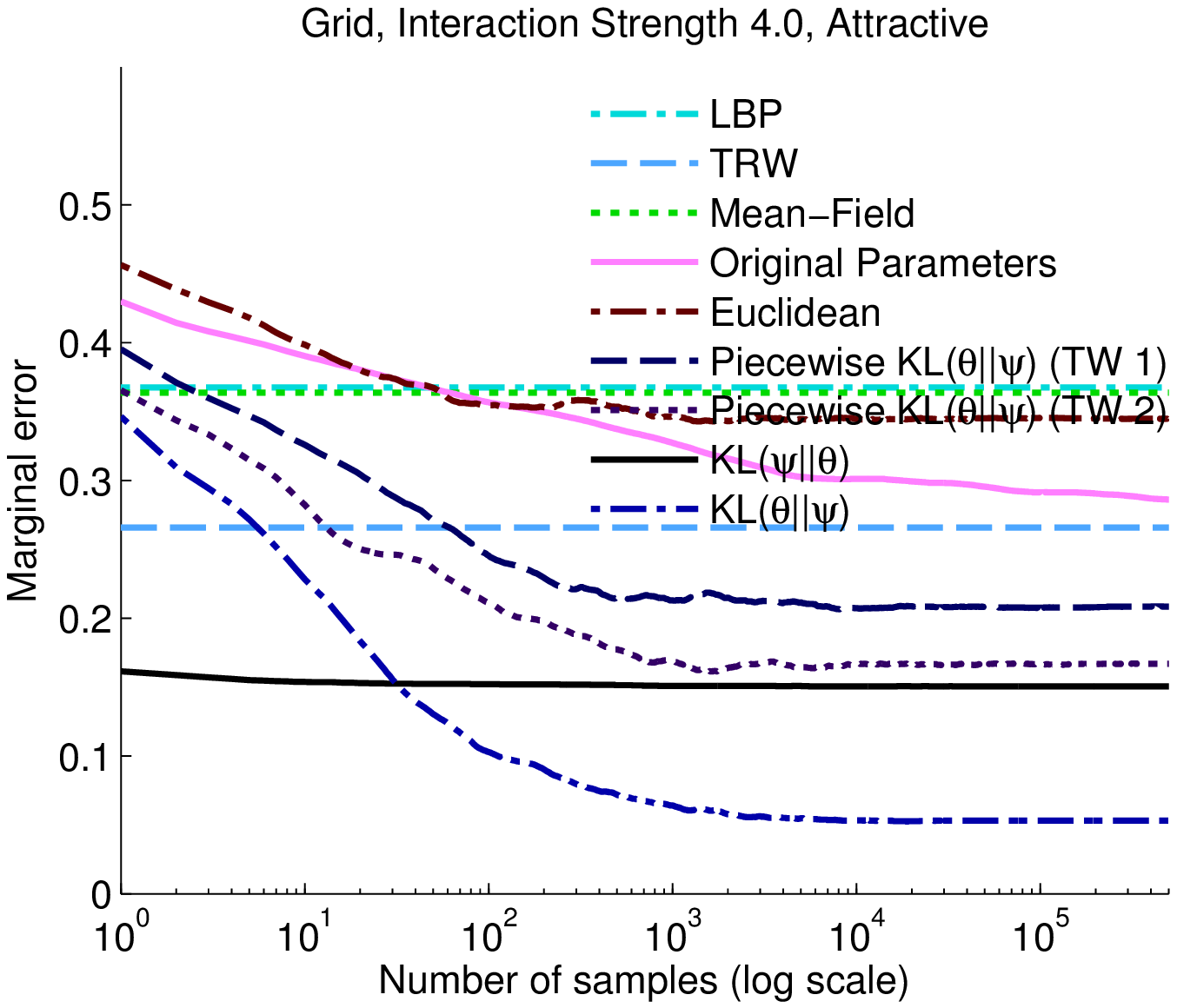}\hfill{}
\par\end{centering}

\begin{centering}
\hfill{}\includegraphics[width=0.445\textwidth]{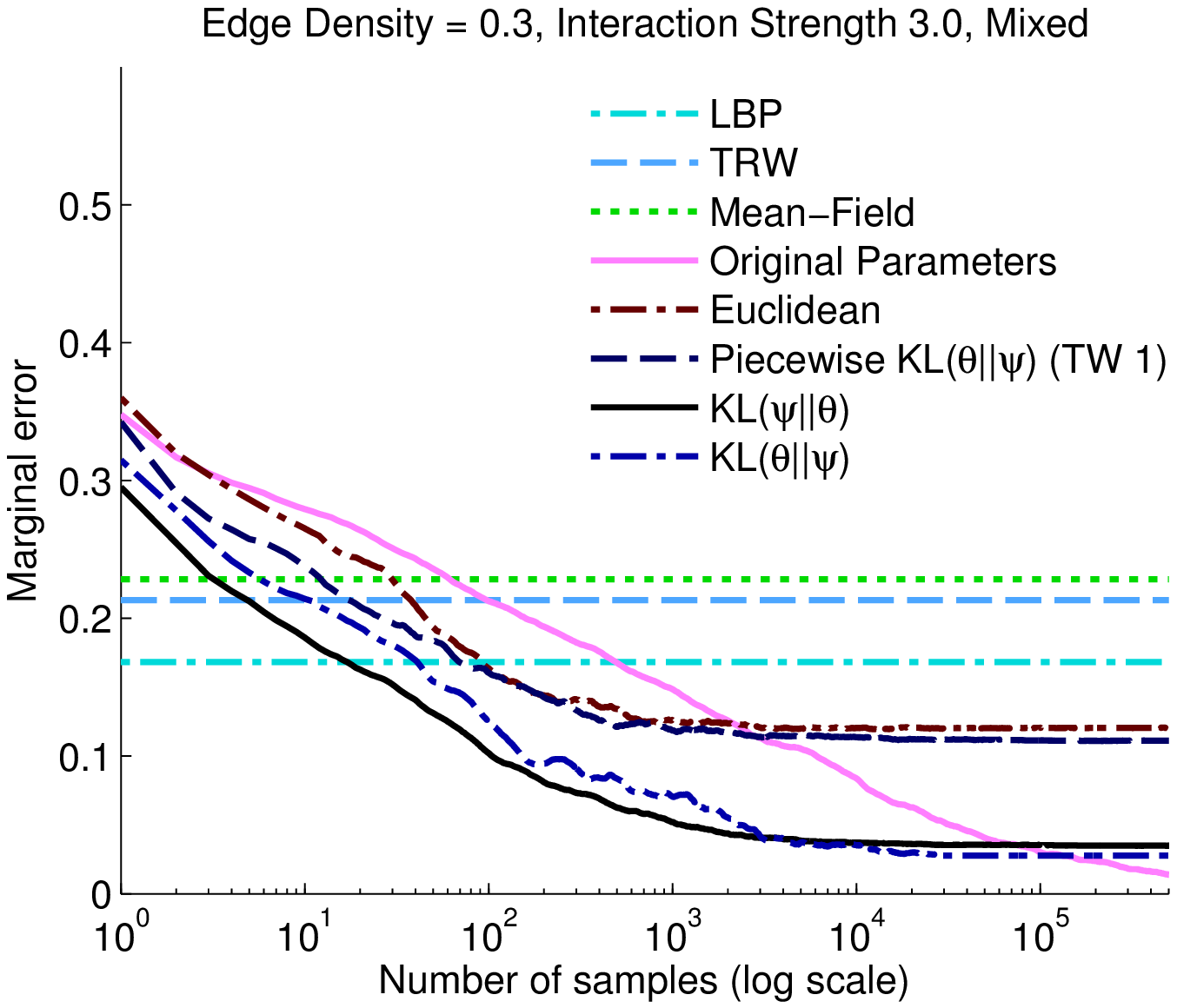}\hfill{}\includegraphics[width=0.445\textwidth]{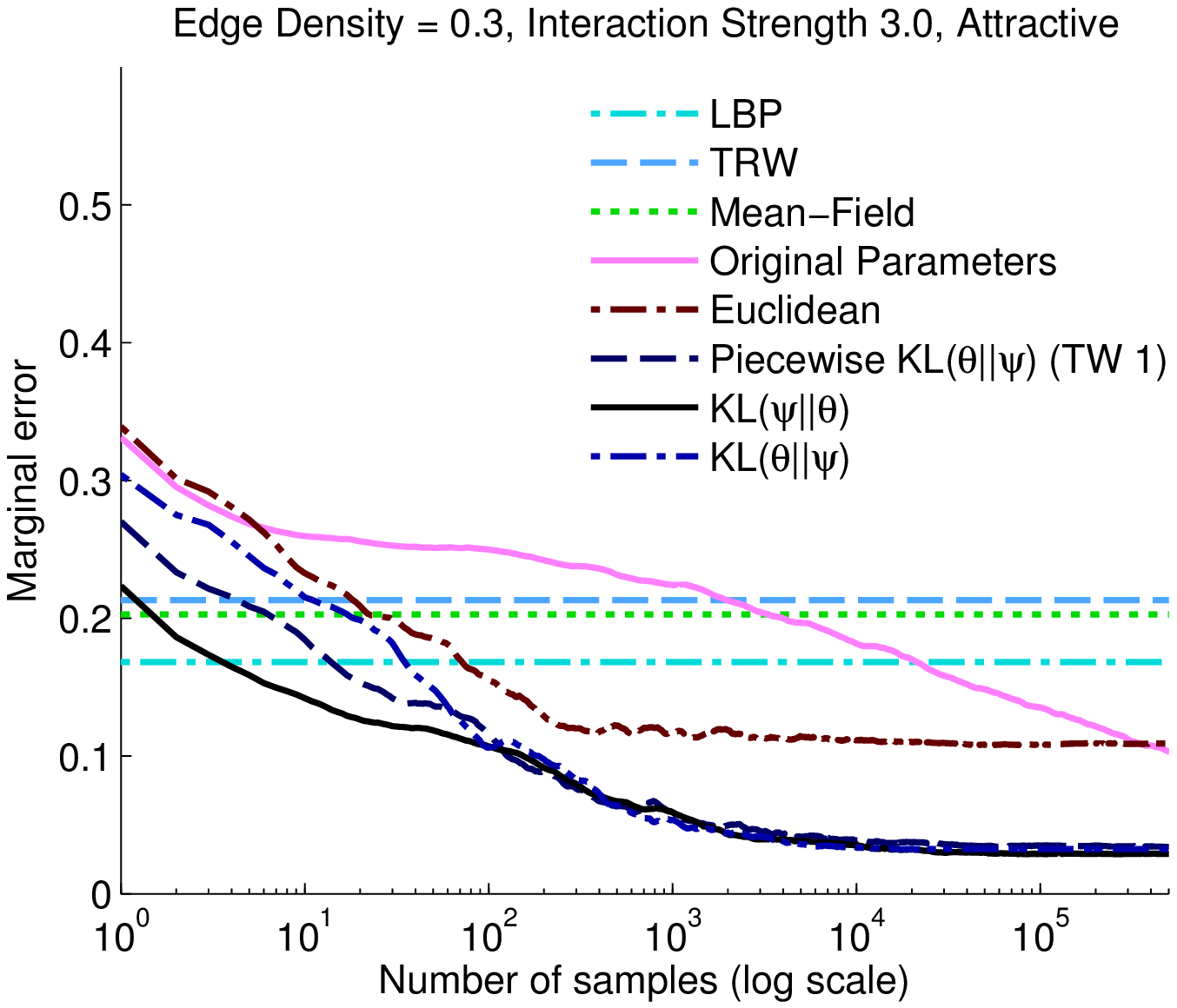}\hfill{}\vspace{-10pt}

\par\end{centering}

\caption{Example plots of the accuracy of obtained marginals vs. the number
of samples. Top: Grid graphs. Bottom: Low-Density Random graphs. (Full
results in appendix.)}
\end{figure}
\begin{table}
\setlength\tabcolsep{4pt}

\begin{tabular}{|c|cc|cc|cc|}
\cline{2-7} 
\multicolumn{1}{c|}{} & \multicolumn{2}{c|}{{\small{Grid, Strength 1.5}}} & \multicolumn{2}{c|}{{\small{Grid, Strength 3}}} & \multicolumn{2}{c|}{{\small{Random Graph, Strength 3.}}}\tabularnewline
\multicolumn{1}{c|}{} & {\small{Gibbs Steps}} & {\small{SVDs}} & {\small{Gibbs Steps}} & {\small{SVDs}} & {\small{Gibbs Steps}} & {\small{SVDs}}\tabularnewline
\hline 
{\small{30,000 Gibbs steps}} & {\small{30k / 0.17s}} &  & {\small{30k / 0.17s}} &  & {\small{30k / 0.04s}} & \tabularnewline
{\small{250,000 Gibbs steps}} & {\small{250k / 1.4s}} &  & {\small{250k / 1.4s}} &  & {\small{250k / 0.33s}} & \tabularnewline
{\small{Euclidean Projection}} &  & {\small{22 / 0.04s}} &  & {\small{78 / 0.15s}} &  & {\small{17 / .0002s}}\tabularnewline
{\small{Piecewise-1 Projection}} &  & {\small{322 / 0.61s}} &  & {\small{547 / 1.0s}} &  & {\small{408 / 0.047s}}\tabularnewline
{\small{KL Projection}} & {\small{30k / 0.17s}} & {\small{265 / 0.55s}} & {\small{30k / 0.17s}} & {\small{471 / 0.94s}} & {\small{30k / 0.04s}} & {\small{300 / 0.037s}}\tabularnewline
\hline 
\end{tabular}

\caption{Running Times on various attractive graphs, showing the number of
Gibbs passes and Singular Value Decompositions, as well as the amount
of computation time. The random graph is based on an edge density
of 0.7. Mean-Field, Loopy BP, and TRW take less than 0.01s.\label{tab:Running-Times}}
\end{table}

\section{Discussion\label{sec:Discussion}}

Inference in high-treewidth graphical models is intractable, which
has motivated several classes of approximations based on tractable
families. In this paper, we have proposed a new notion of ``tractability'',
insisting not that a graph has a fast algorithm for exact inference,
but only that it obeys parameter-space conditions ensuring that Gibbs
sampling will converge rapidly to the stationary distribution. For
the case of Ising models, we use a simple condition that can guarantee
rapid mixing, namely that the spectral norm of the matrix of interaction
strengths is less than one.

Given an intractable set of parameters, we consider using this approximate
family by ``projecting'' the intractable distribution onto it under
several divergences. First, we consider the Euclidean distance of
parameters, and derive a dual algorithm to solve the projection, based
on an iterative thresholding of the singular value decomposition.
Next, we extend this to more probabilistic divergences. Firstly, we
consider a novel ``piecewise'' divergence, based on computing the
exact KL-divergnce on several low-treewidth subgraphs. Secondly, we
consider projecting onto the KL-divergence. This requires a stochastic
approximation approach where one repeatedly generates samples from
the model, and projects in the Euclidean norm after taking a gradient
step.

We compare experimentally to Gibbs sampling on the original parameters,
along with several standard variational methods. The proposed methods
are more accurate than variational approximations. Given enough time,
Gibbs sampling using the original parameters will always be more accurate,
but with finite time, projecting onto the fast-mixing set to generally
gives better results.

Future work might extend this approach to general Markov random fields.
This will require two technical challenges. First, one must find a
bound on the dependency matrix for general MRFs, and secondly, an
algorithm is needed to project onto the fast-mixing set defined by
this bound. Fast-mixing distributions might also be used for learning.
E.g., if one is doing maximum likelihood learning using MCMC to estimate
the likelihood gradient, it would be natural to constrain the parameters
to a fast mixing set.

One weakness of the proposed approach is the apparent looseness of
the spectral norm bound. For the two dimensional Ising model with
no univariate terms, and a constant interaction strength $\beta$,
there is a well-known threshold $\beta_{c}=\frac{1}{2}\ln(1+\sqrt{2})\approx.4407$,
obtained using more advanced techniques than the spectral norm \cite{CriticalIsingMixing}.
Roughly, for $\beta<\beta_{c}$, mixing is known to occur quickly
(polynomial in the grid size) while for $\beta>\beta_{c}$, mixing
is exponential. On the other hand, the spectral bound norm will be
equal to one for $\beta=.25$, meaning the bound is too conservative
in this case by a factor of $\beta_{c}/.25\approx1.76$. A tighter
bound on when rapid mixing will occur would be more informative.

\newpage{}

\bibliographystyle{plain}
\bibliography{thebib,bibliography_pamipaper,bibliography_singleloop}

\newpage{}

\section*{Appendix}

Recall that we are interested in the minimization

\begin{eqnarray}
\min_{B,D} &  & ||A-B||_{F}\label{eq:sparse-projection-minimization-1}\\
s.t. &  & ||D||_{2}\leq c\nonumber \\
 &  & Z_{ij}D_{ij}=0\nonumber \\
 &  & D=|B|.\nonumber 
\end{eqnarray}

\begin{lem}
\label{lem:projection-minimization-2}If we define $R=|A|$, this
is equivalent to the minimization

\begin{eqnarray}
\min_{D} &  & ||R-D||_{F}\label{eq:sparse-projection-minimization-1-1}\\
s.t. &  & ||D||_{2}\leq c\nonumber \\
 &  & Z_{ij}D_{ij}=0\nonumber \\
 &  & D\geq0\nonumber 
\end{eqnarray}
\end{lem}
\begin{proof}
For fixed $D$, the minimum $B$ will always be achieved by $B=D\odot\text{sign}(A)$,
meaning $||A-B||_{F}=||A-D\odot\text{sign}(A)||_{F}=||R-D||_{F}$.
\end{proof}
To actually project the parameters $A=(\beta_{ij})$ corresponding
to an Ising model, one first takes the absolute value $R=|A|$, and
passes it as input to this minimization. After finding the minimizing
argument, the new parameters are $B=D\odot\text{sign}(A)$.
\begin{thm}
\label{thm:projection_theorem-1-1}Define $R=|A|$. The minimization
in Eq. \ref{eq:sparse-projection-minimization-singleline} is equivalent
to the problem of $\max_{M\geq0,\Lambda}g(\Lambda,M)$, where the
objective and gradient of $g$ are, for $D(\Lambda,M)=\Pi_{c}[R+M-\Lambda\odot Z],$

\begin{align}
g(\Lambda,M) & =\frac{1}{2}||D(\Lambda,M)-R||_{F}^{2}+\Lambda\cdot Z\cdot D(\Lambda,M) - M \cdot D(\Lambda,M)\label{eq:sparse-projection-dual-obj-appendix}\\
\frac{dg}{d\Lambda} & =Z\odot D(\Lambda,M)\label{eq:sparse-projection-dual-grad-1-appendix}\\
\frac{dg}{dM} & = -D(\Lambda,M).\label{eq:sparse-projection-dual-grad-2-appendix}
\end{align}
\end{thm}
\begin{proof}
The minimization in Eq. \ref{eq:sparse-projection-minimization-1-1}
has the Lagrangian
\begin{equation}
\mathcal{L}(D,\Lambda,M)=\frac{1}{2}||D-R||_{F}^{2}+I[||D||_{2}\leq c]+\Lambda\cdot Z\cdot D-M\cdot D,\label{eq:sparse-projection-lagrangian}
\end{equation}
where $I$ is an indicator function returning $\infty$ if $||D||_{2}>c$
and zero otherwise, $\Lambda$ is a matrix of Lagrange multipliers
enforcing that $Z_{ij}D_{ij}=0$, and $M$ is a matrix of Lagrange
multipliers enforcing that $D\geq0$.

Standard duality theory states that Eq. \ref{eq:sparse-projection-minimization-1-1}
is equivalent to the saddle-point problem $\max_{M\geq0,\Lambda}\min_{D}\mathcal{L}(D,\Lambda,M)$.
So, we are interested in evaluating $g(\Lambda,M)=\min_{D}\mathcal{L}(D,\Lambda,M)$
for fixed $\Lambda$ and $M$. Some algebra gives
\begin{multline*}
\arg\min_{D}\mathcal{L}(D,\Lambda,M)\\
=\arg\min_{D}\frac{1}{2}||D-R||_{F}^{2}+\Lambda\cdot Z\cdot D+I\bigl[||D||_{2}\leq c\bigr]-M\cdot D\\
=\arg\min_{D}\frac{1}{2}||D-(R+M-\Lambda\odot Z)||_{F}^{2}+I\bigl[||D||_{2}\leq c\bigr],
\end{multline*}
which shows that $g$ can be calculated as in Eq. \ref{eq:sparse-projection-dual-obj-appendix}.

Next, we are interested in the gradient of $g$. By applying Danskin's
theorem to Eq. \ref{eq:sparse-projection-lagrangian}, we have that
$\frac{d}{dM}\arg\min_{D}\mathcal{L}(D,\Lambda,M)$ will be exactly
$-D$ where $D$ is the minimizer of Eq. \ref{eq:sparse-projection-lagrangian}.
This establishes Eq. \ref{eq:sparse-projection-dual-grad-2-appendix}.
Similarly, it can be shown that $\frac{d}{d\Lambda}\arg\min_{D}\mathcal{L}(D,\Lambda,M)=Z\odot D,$
establishing Eq. \ref{eq:sparse-projection-dual-grad-1-appendix}.\end{proof}
\begin{thm}
The divergence $D(\theta,\psi)=KL(\psi||\theta)$ has the gradient
\[
\frac{d}{d\psi}D(\theta,\psi)=\sum_{x}p(x;\psi)(\psi-\theta)\cdot f(x)\left(f(x)-\mu(\psi)\right).
\]
\end{thm}
\begin{proof}
Firstly, it can be shown that
\[
D(\theta,\psi)=\sum_{x}p(x;\psi)(\psi-\theta)\cdot f(x)+A(\theta)-A(\psi).
\]

From this, it follows that
\begin{alignat*}{1}
\frac{d}{d\psi}D(\theta,\psi)= & \sum_{x}\frac{dp(x;\psi)}{d\psi}(\psi-\theta)\cdot f(x)\\
+ & \sum_{x}p(x;\psi)f(x)-\mu(\psi).
\end{alignat*}
This can be seen to be equivalent to the result by observing that
the second two terms cancel, and that $dp(x;\psi)/d\psi=p(x;\psi)(f(x)-\mu(\psi))$.
\end{proof}
\begin{figure}
\begin{centering}
\includegraphics[width=0.5\textwidth]{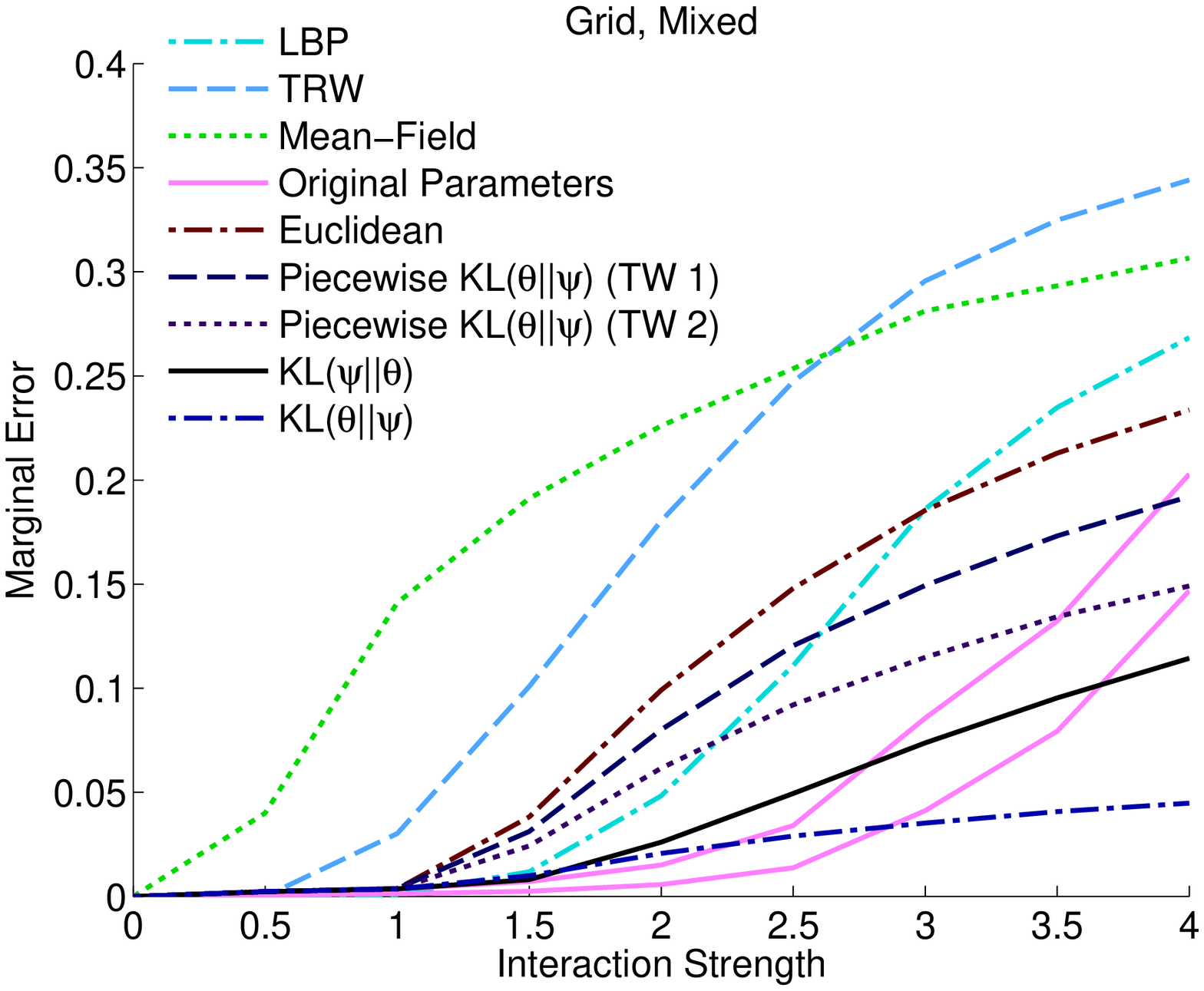}\includegraphics[width=0.5\textwidth]{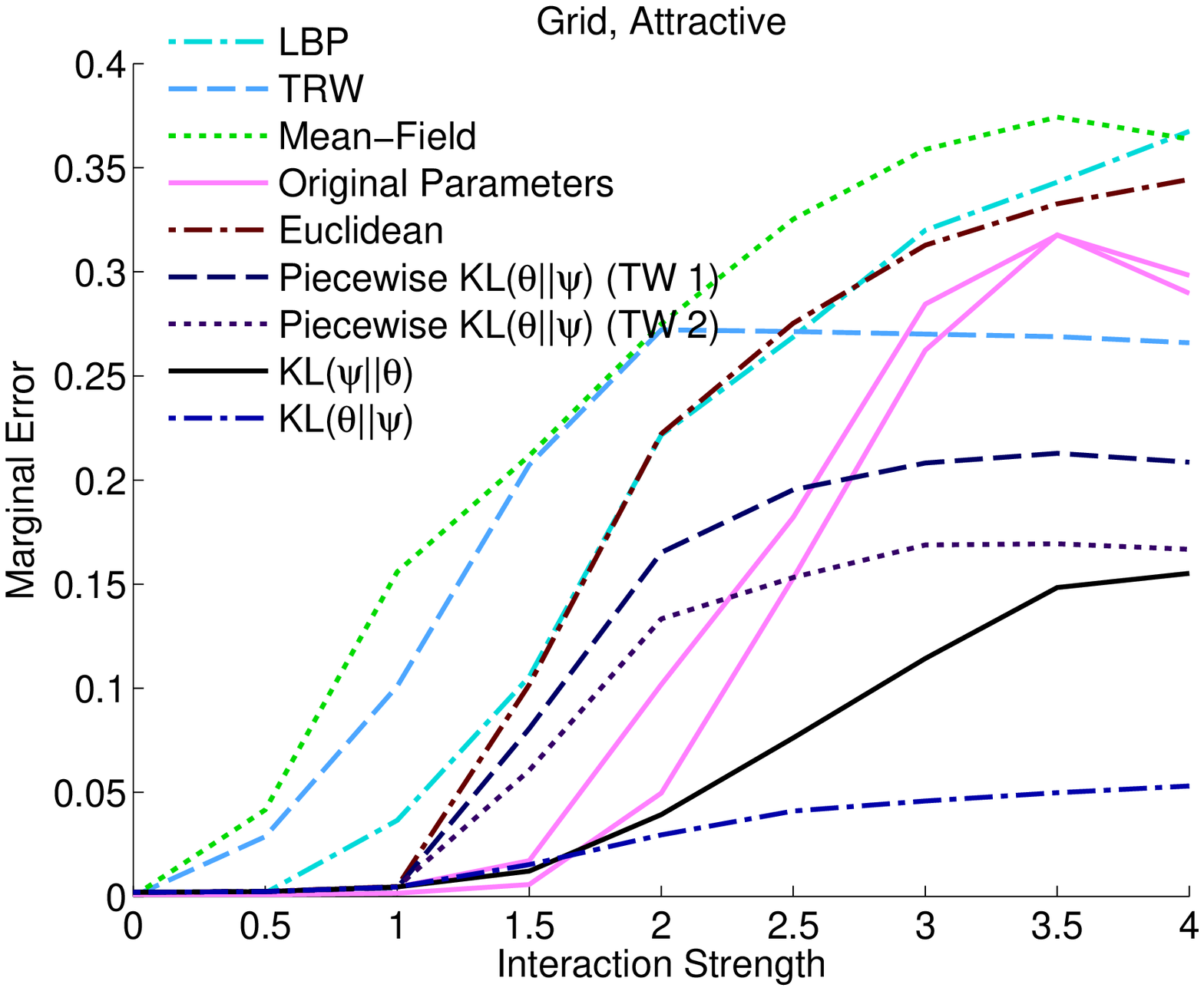}
\par\end{centering}

\caption{Accuracy on Grids, as a function of edge strength. All sampling methods
use 30k samples, except sampling on the original parameters which
includes a second (lower) curve with 250k samples.}

\end{figure}

\begin{figure}
\begin{centering}
\includegraphics[width=0.5\textwidth]{esacc_plot/rg_1_0\lyxdot 3}\includegraphics[width=0.5\textwidth]{esacc_plot/rg_a1_0\lyxdot 3}
\par\end{centering}

\begin{centering}
\includegraphics[width=0.5\textwidth]{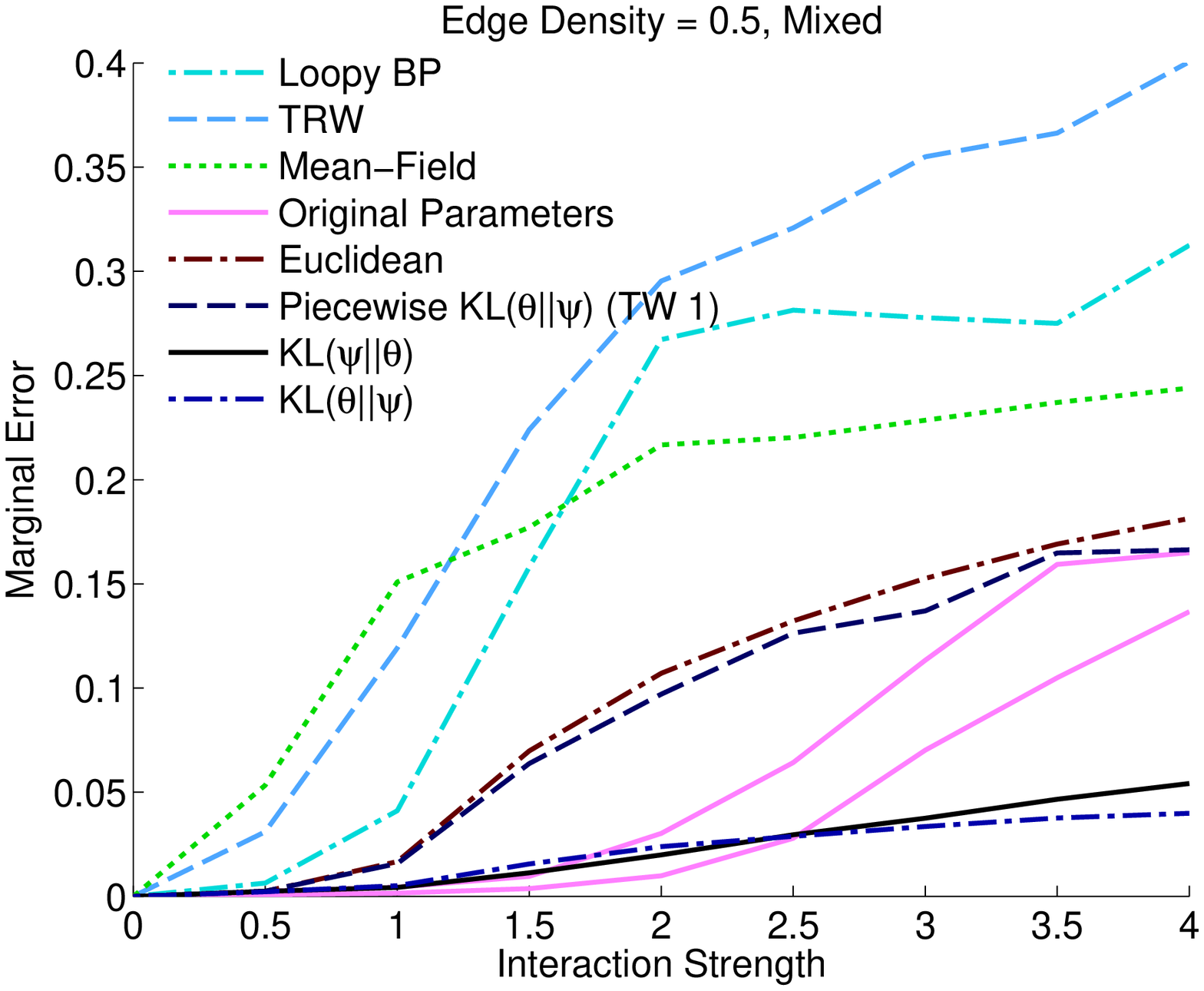}\includegraphics[width=0.5\textwidth]{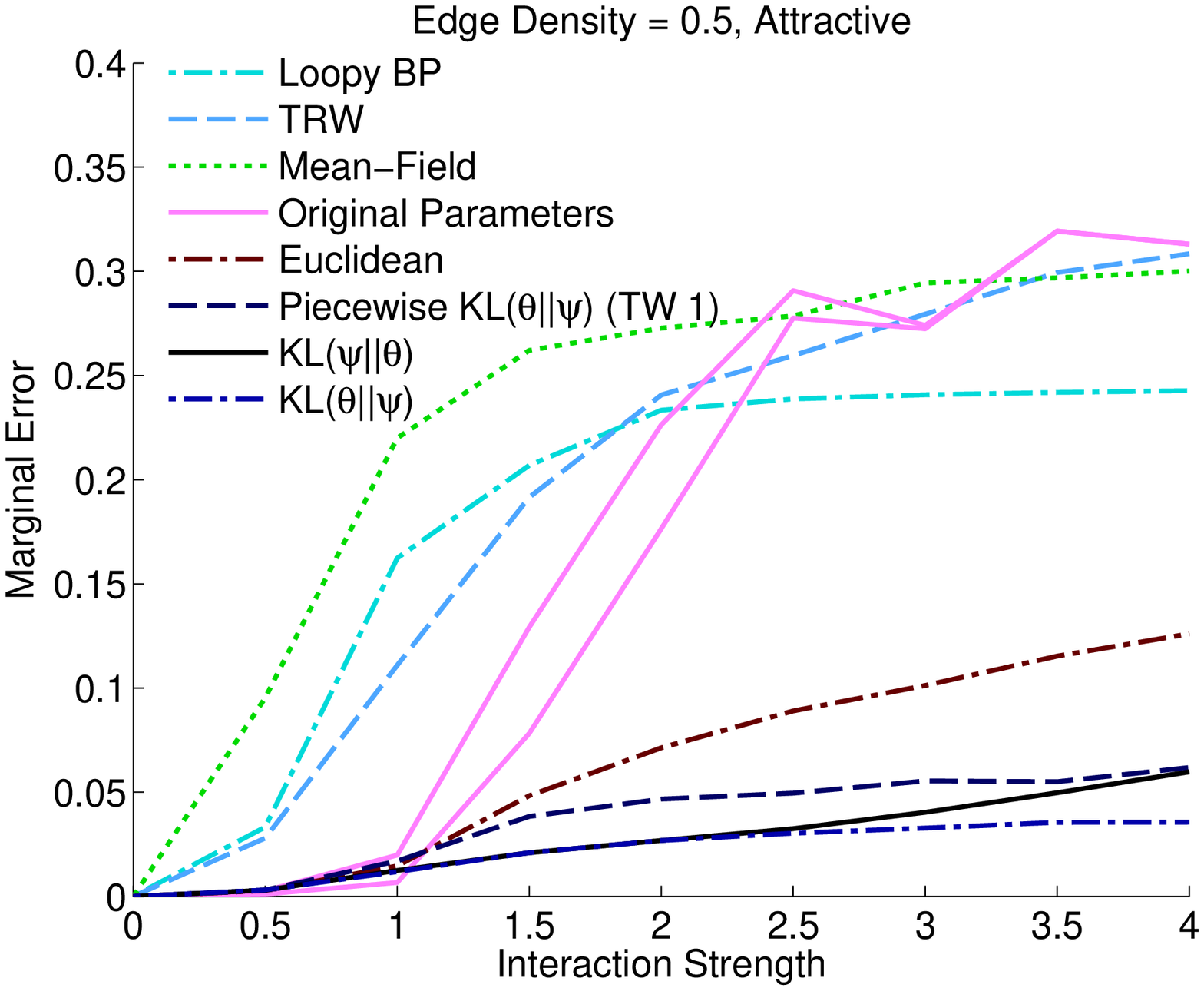}
\par\end{centering}

\begin{centering}
\includegraphics[width=0.5\textwidth]{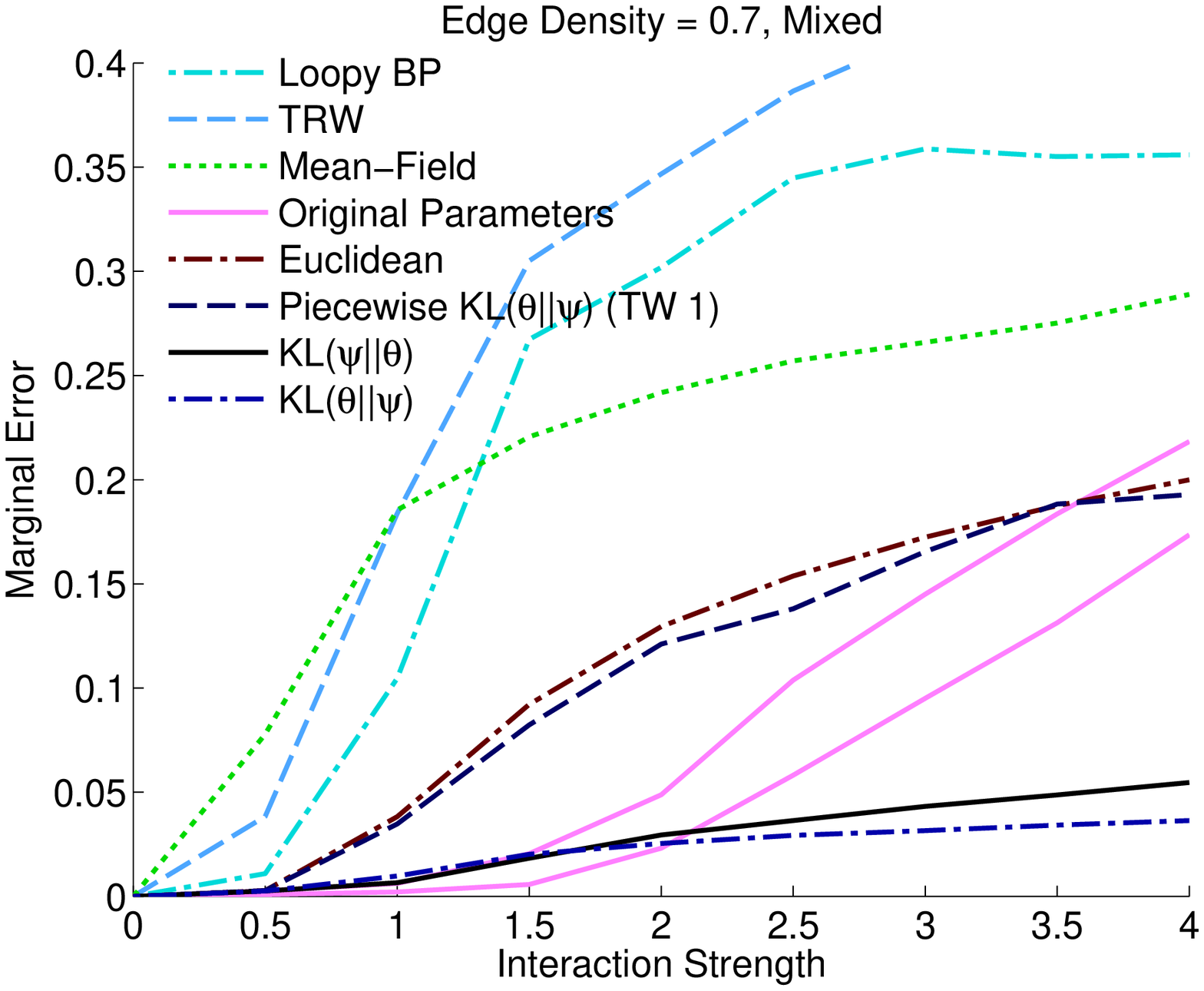}\includegraphics[width=0.5\textwidth]{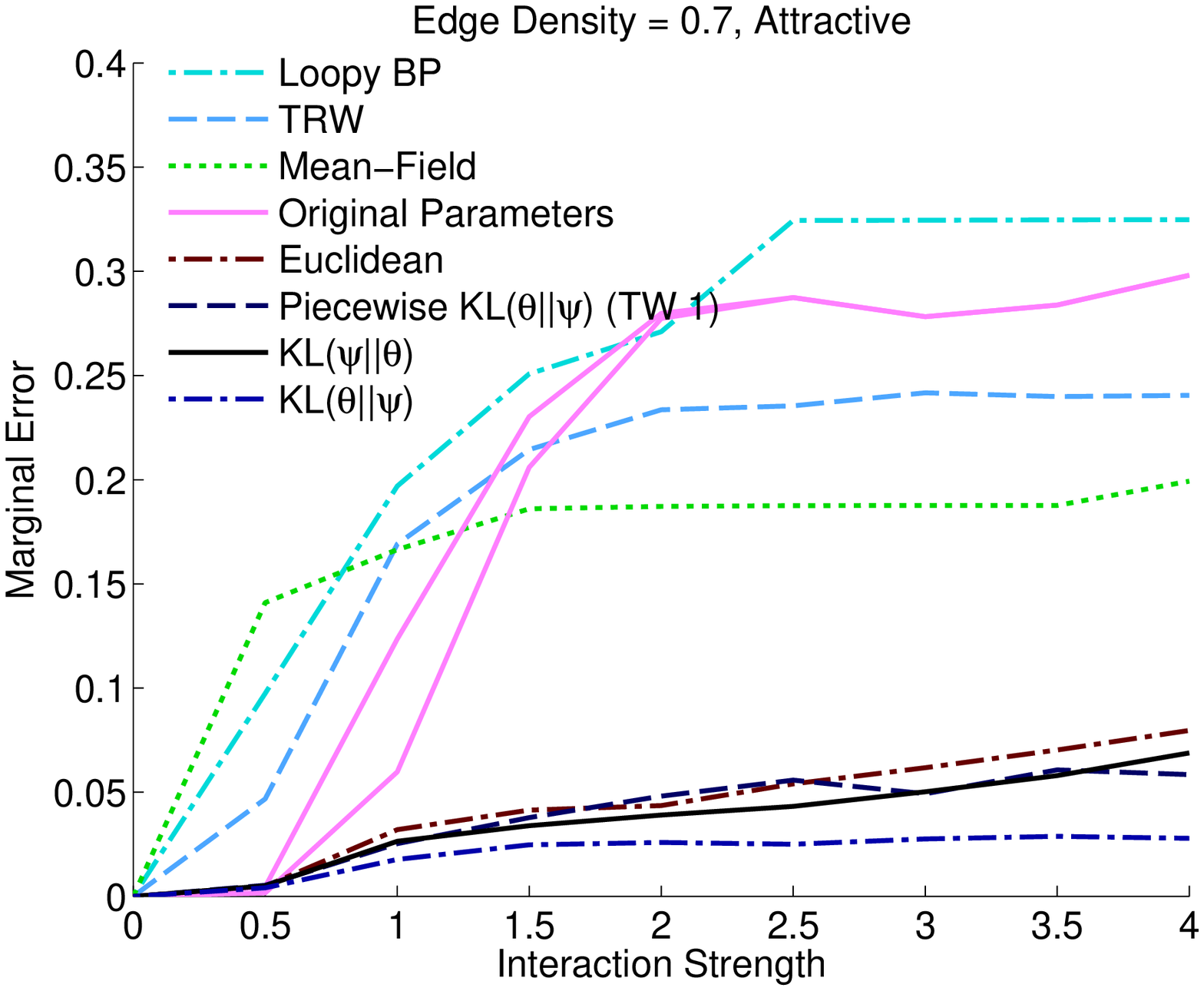}
\par\end{centering}

\caption{Accuracy on random graphs, as a function of edge strength. All sampling
methods use 30k samples, except sampling on the original parameters
which includes a second (lower) curve with 250k samples.}
\end{figure}

\begin{figure}
\begin{centering}
\includegraphics[width=0.5\textwidth]{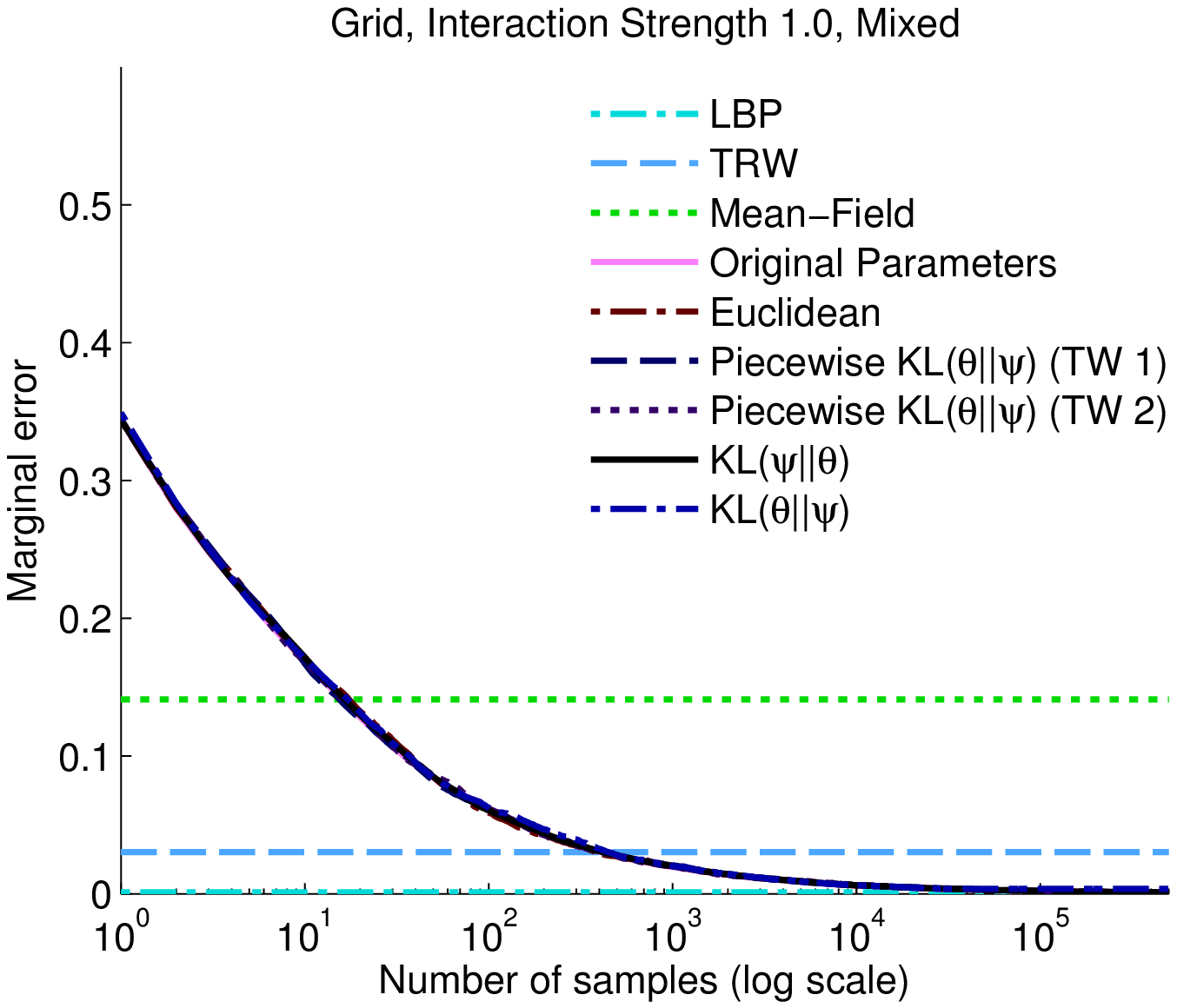}\includegraphics[width=0.5\textwidth]{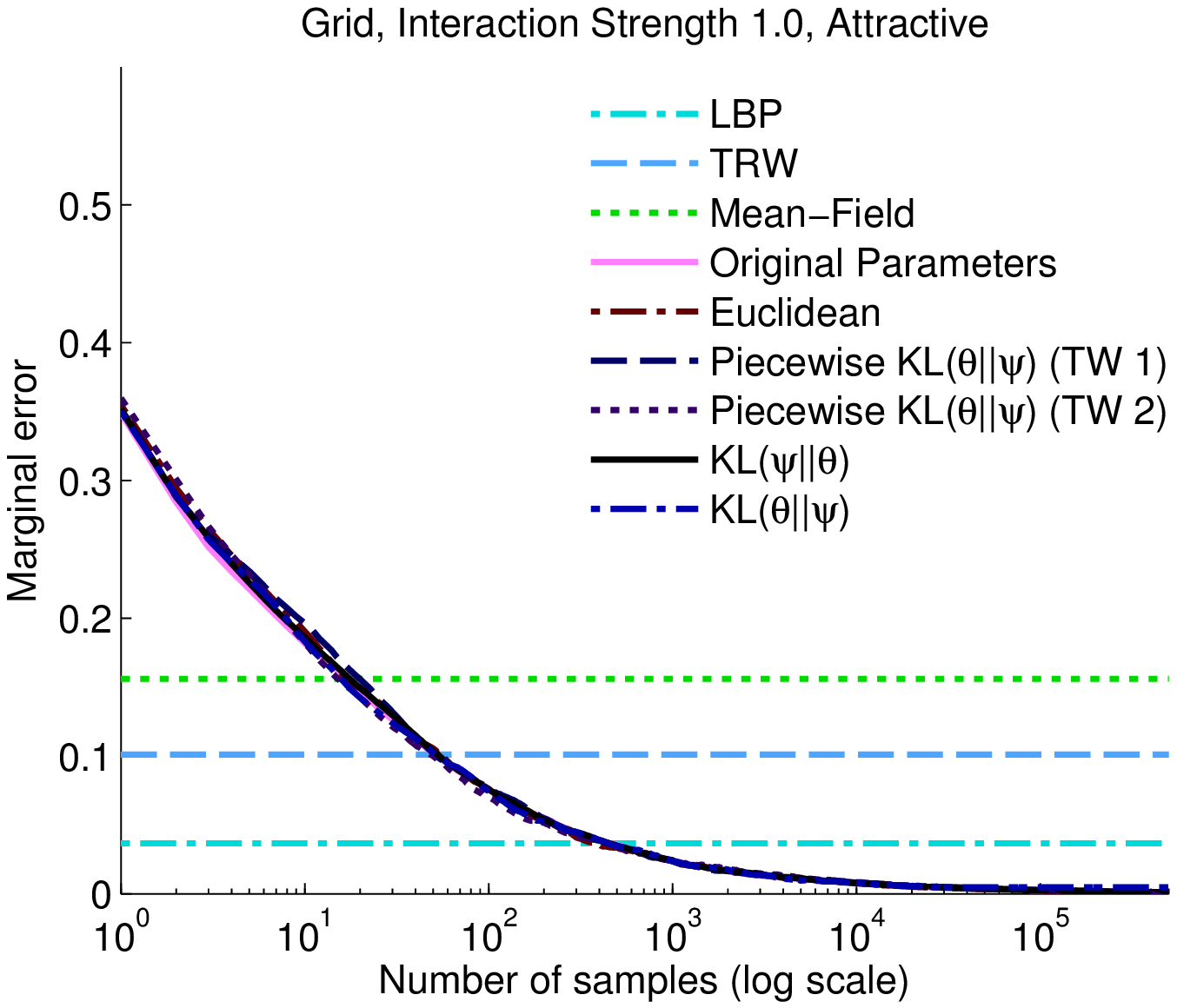}
\par\end{centering}

\begin{centering}
\includegraphics[width=0.5\textwidth]{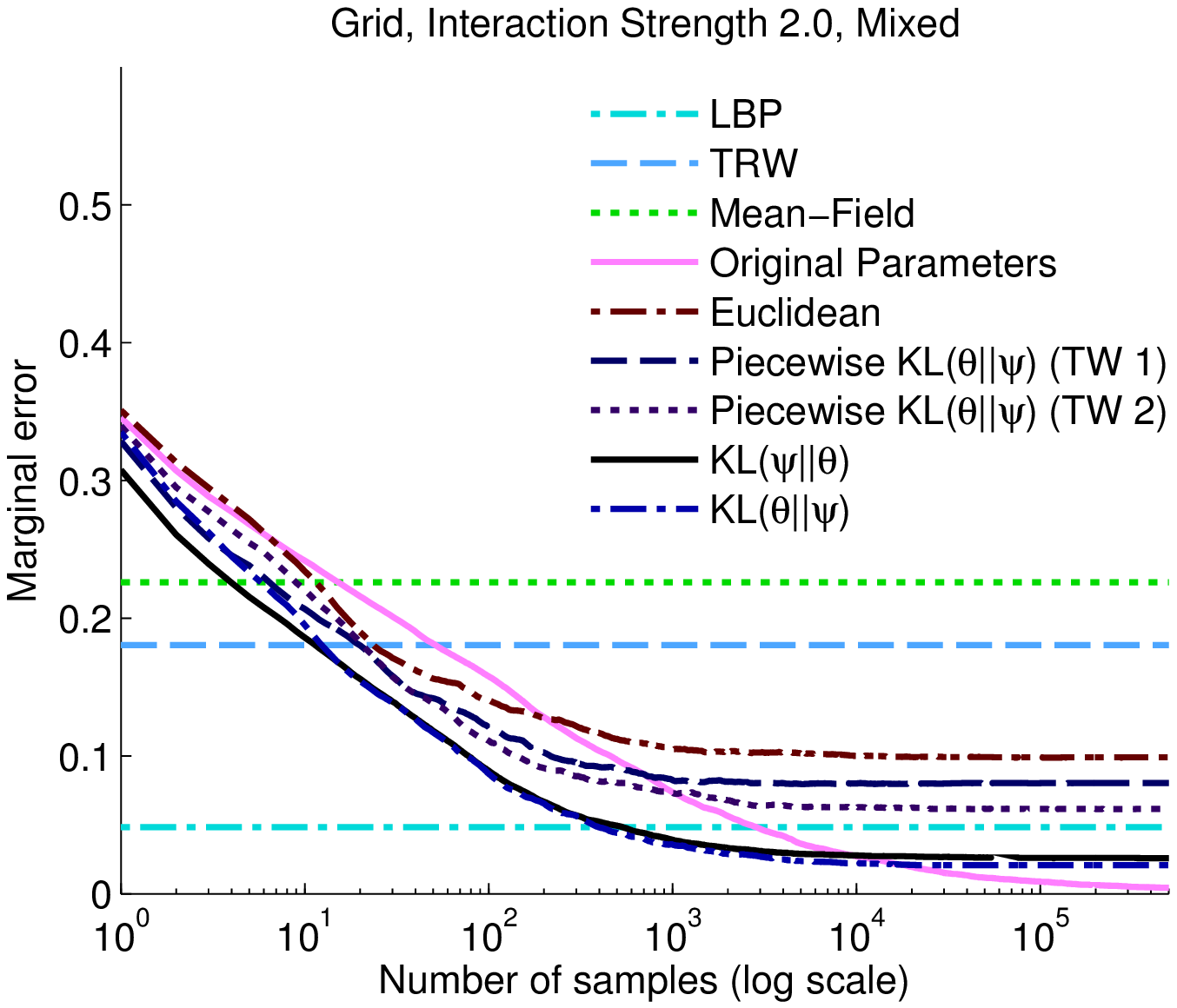}\includegraphics[width=0.5\textwidth]{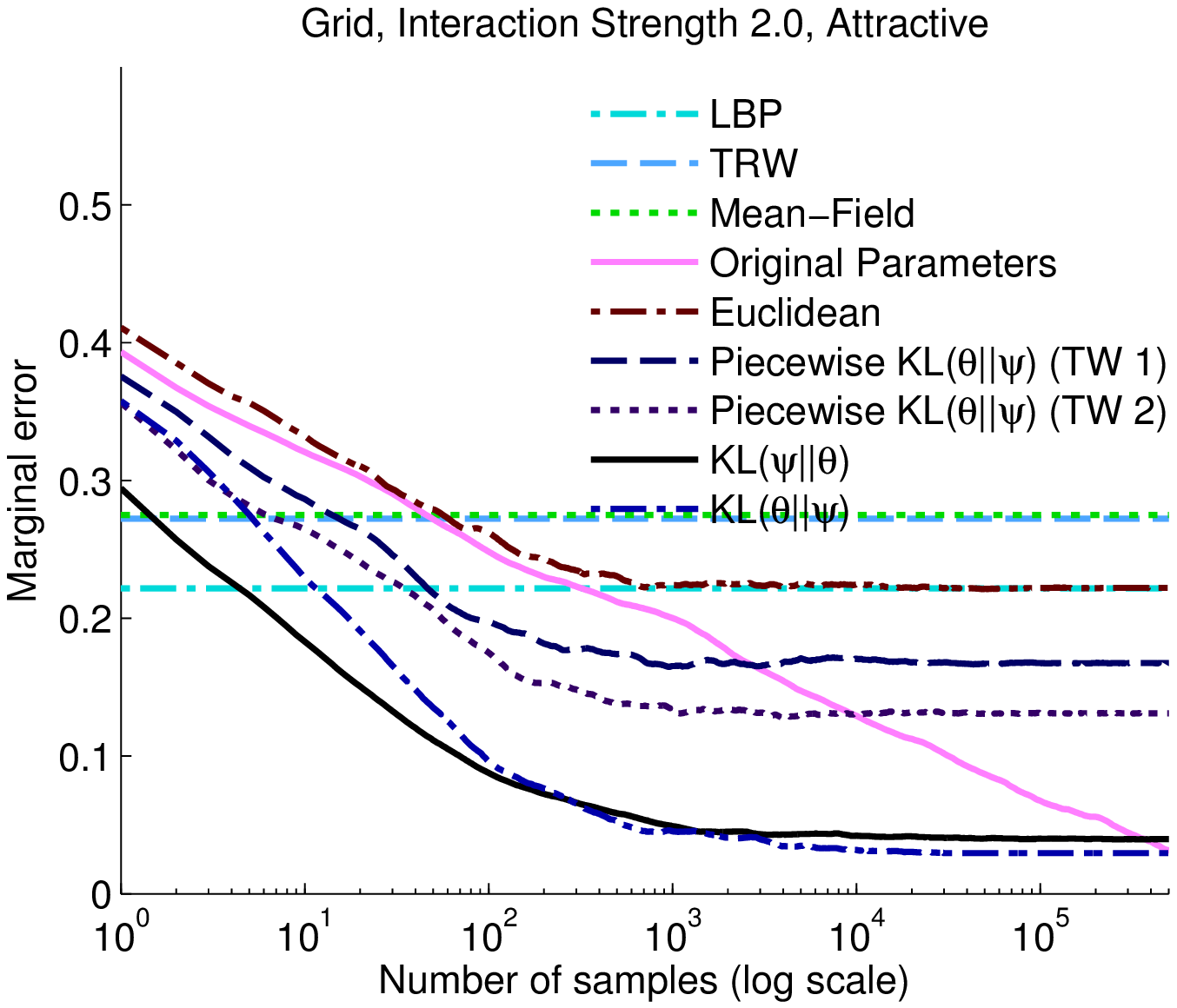}
\par\end{centering}

\begin{centering}
\includegraphics[width=0.5\textwidth]{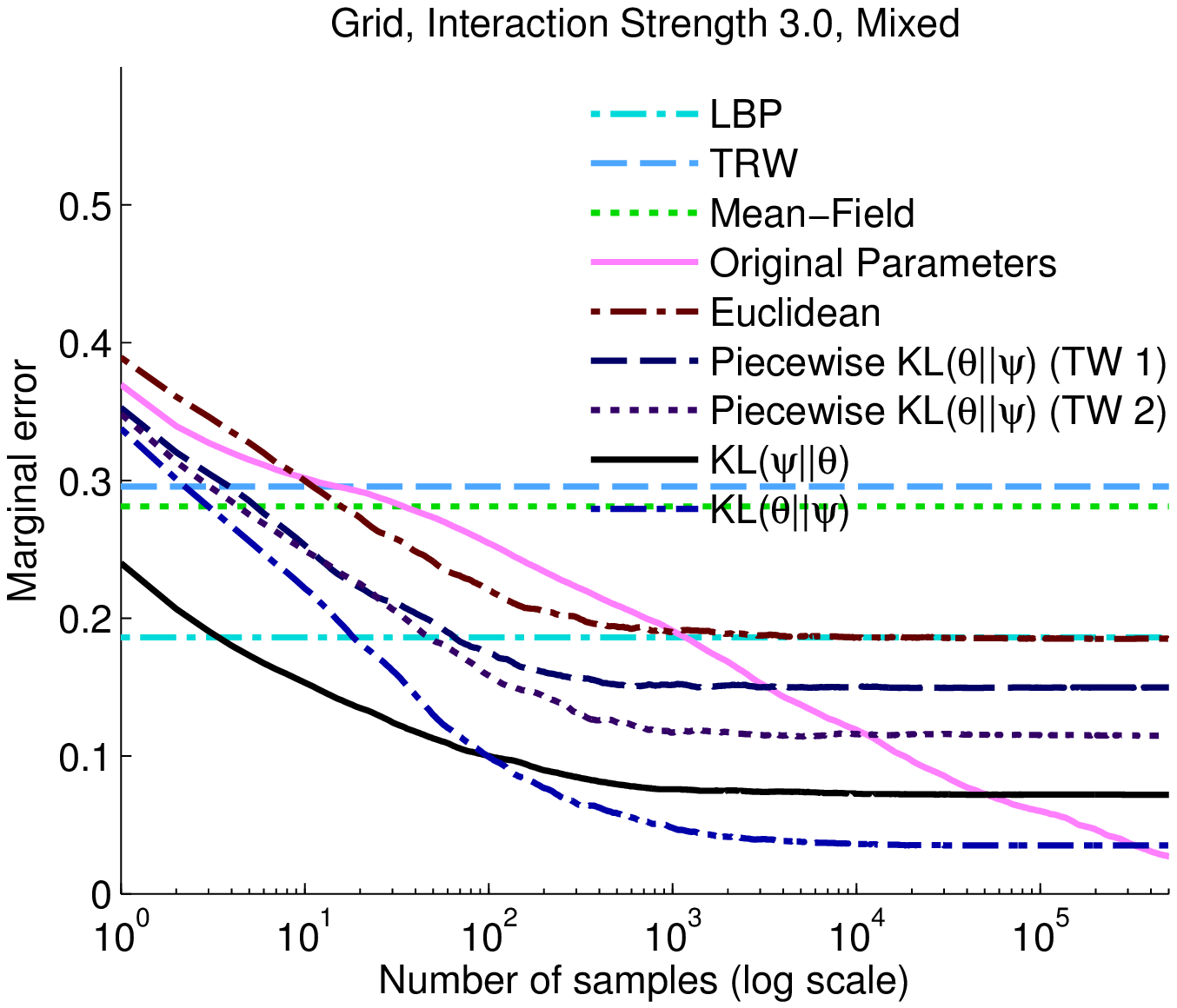}\includegraphics[width=0.5\textwidth]{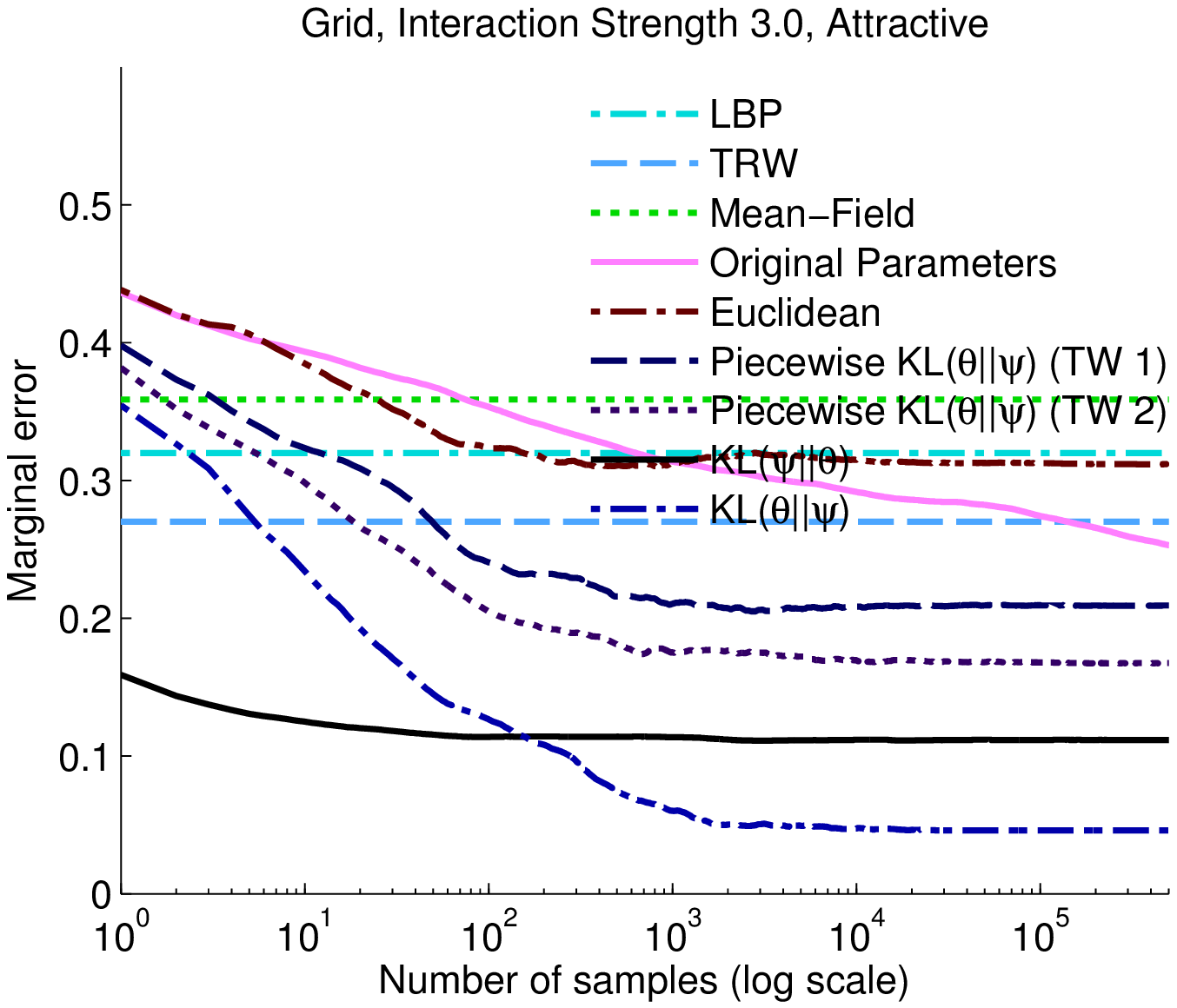}
\par\end{centering}

\begin{centering}
\includegraphics[width=0.5\textwidth]{tmacc_grid_plot/tmacc_1\lyxdot 00_4\lyxdot 00}\includegraphics[width=0.5\textwidth]{tmacc_grid_plot/tmacc_a_1\lyxdot 00_4\lyxdot 00}
\par\end{centering}

\caption{Accuracy on Grids as a function of time}
\end{figure}

\begin{figure}
\begin{centering}
\includegraphics[width=0.5\textwidth]{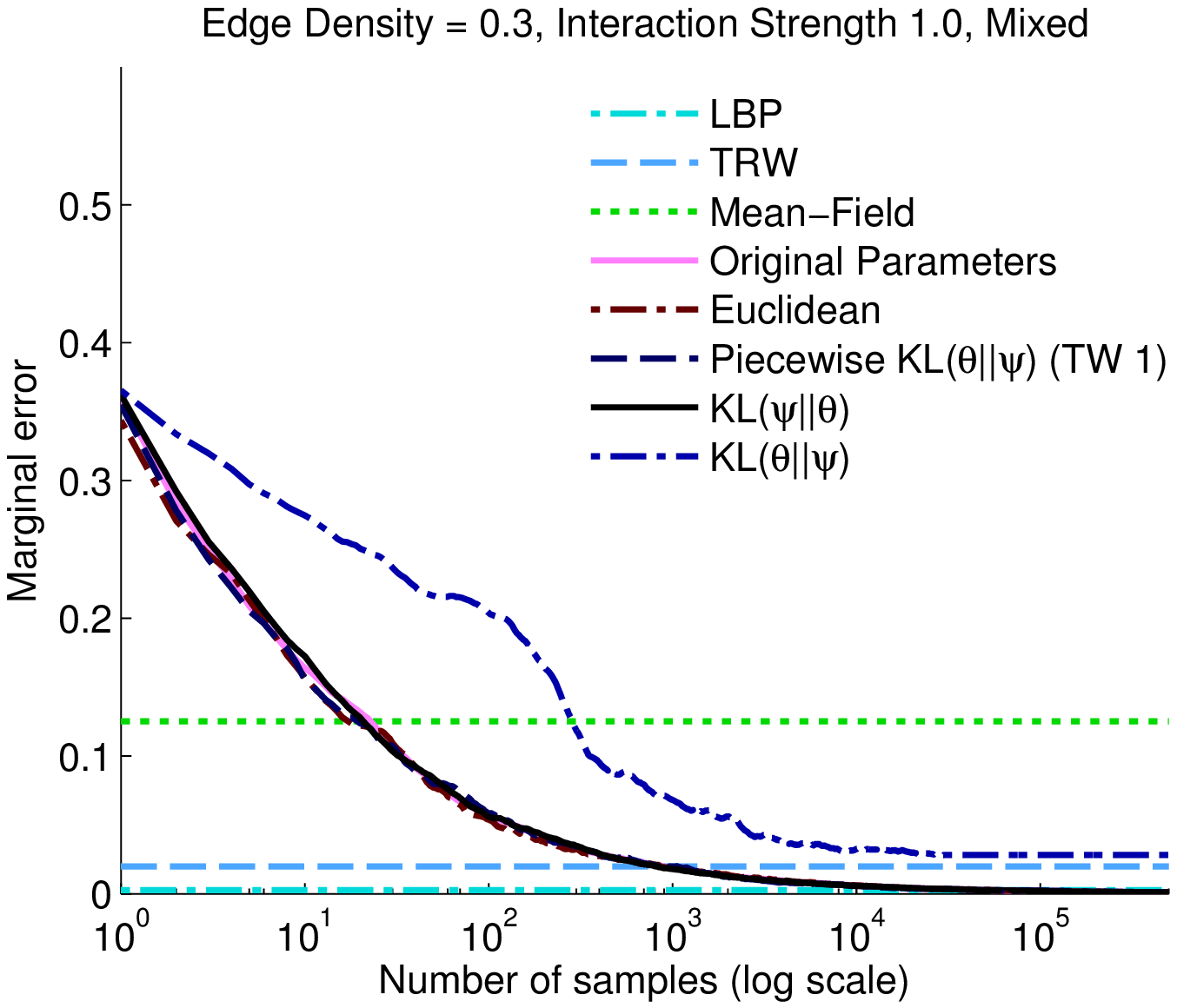}\includegraphics[width=0.5\textwidth]{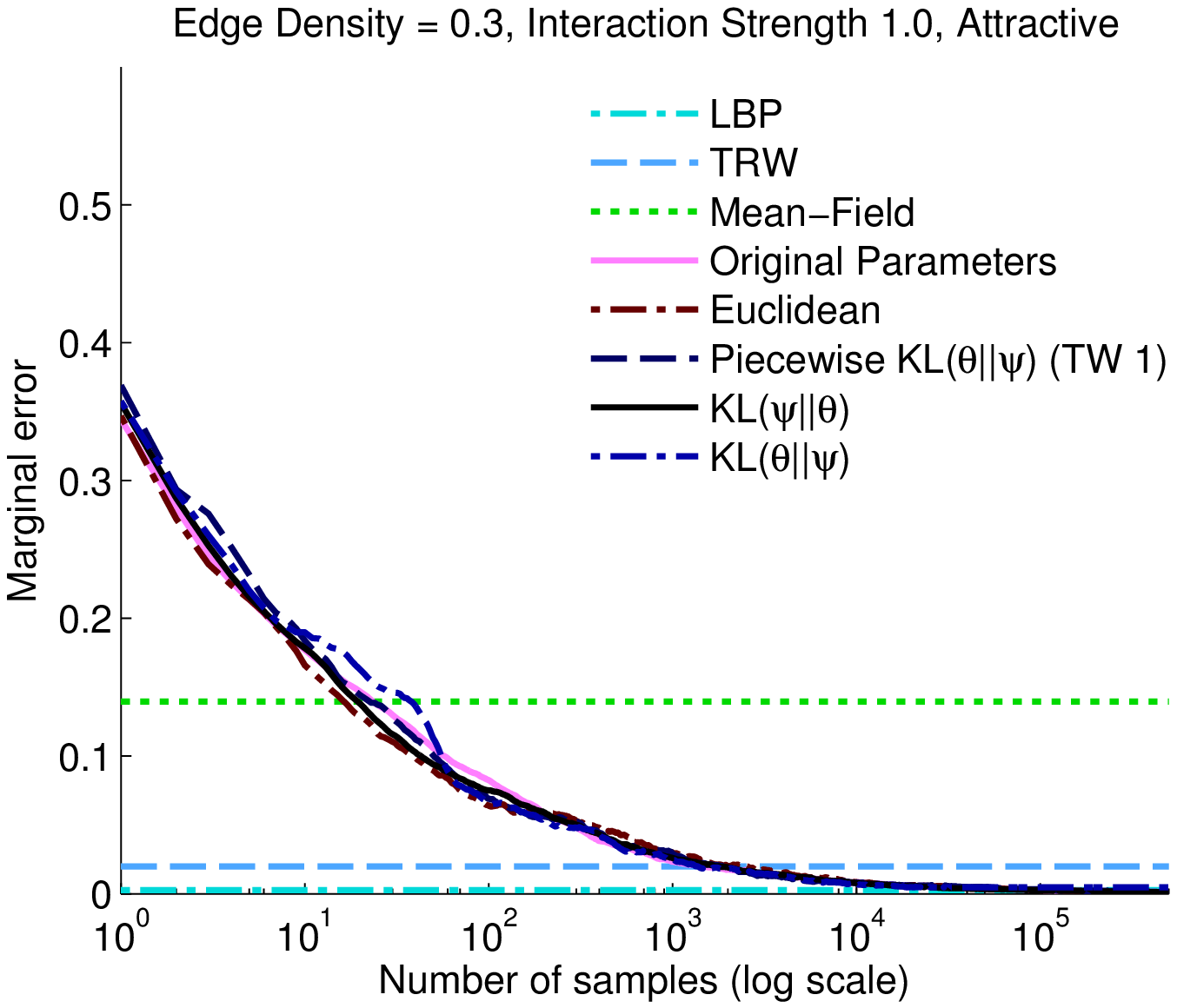}
\par\end{centering}

\begin{centering}
\includegraphics[width=0.5\textwidth]{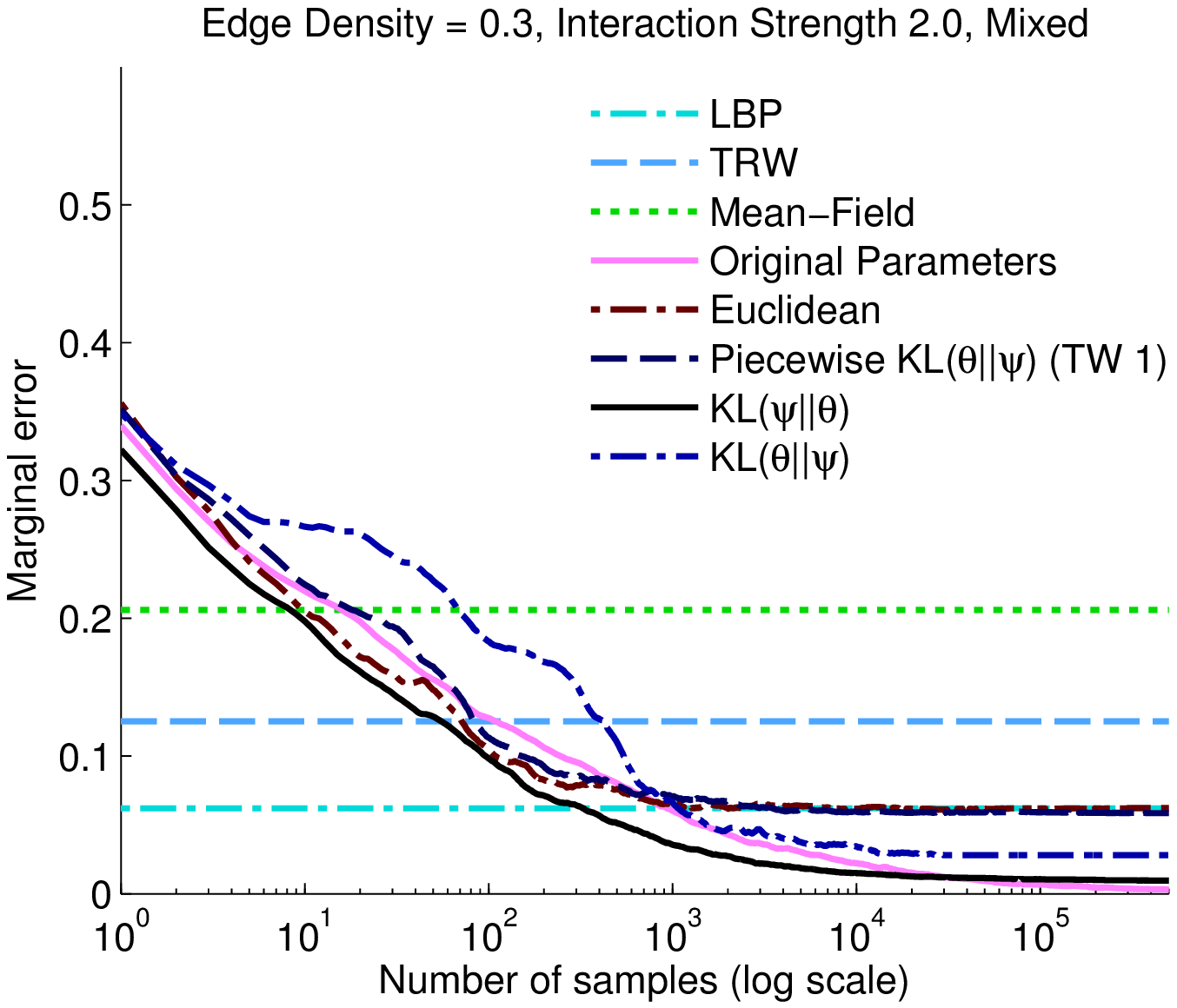}\includegraphics[width=0.5\textwidth]{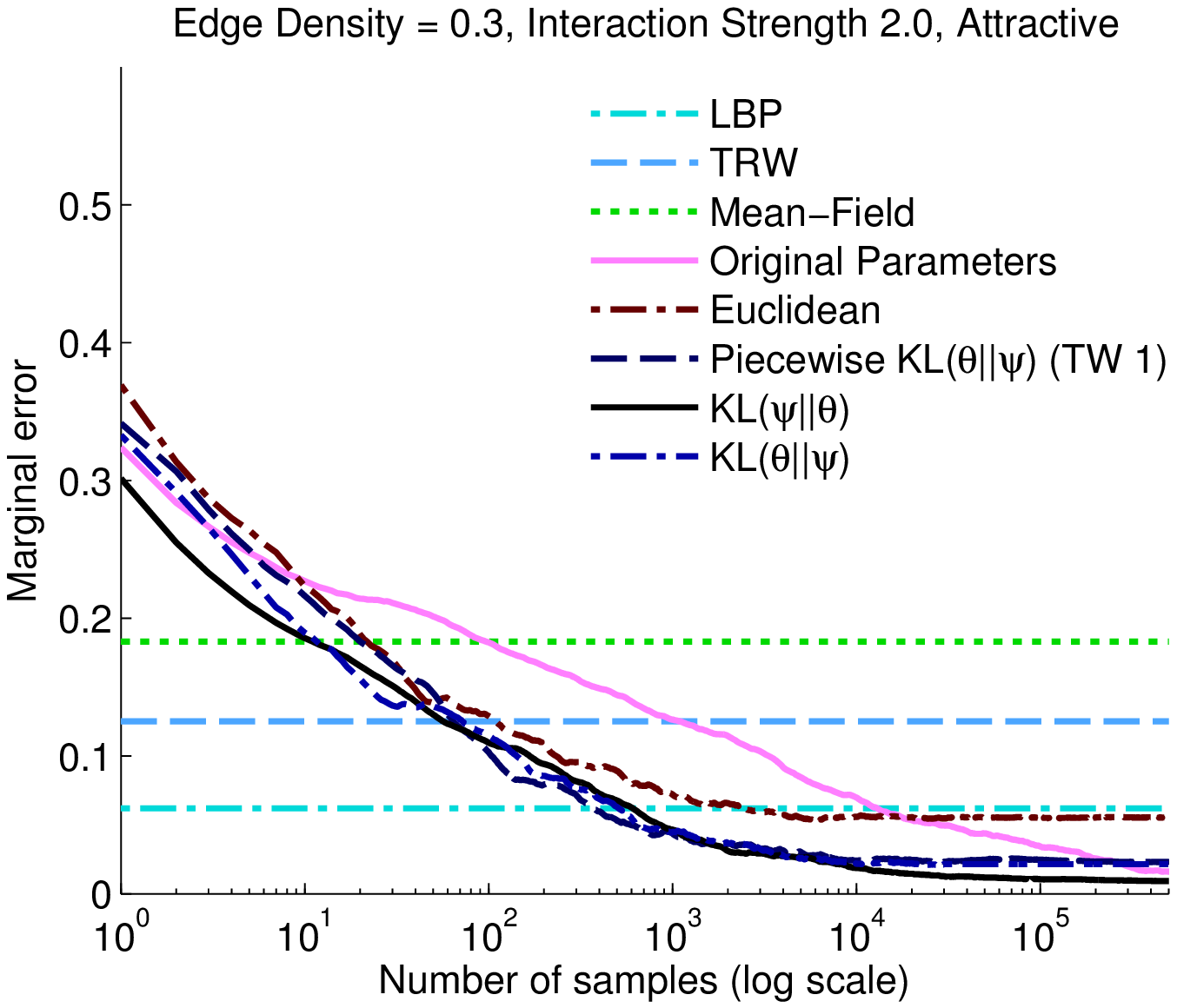}
\par\end{centering}

\begin{centering}
\includegraphics[width=0.5\textwidth]{tmacc_rg_plot/tmacc_rnd_0\lyxdot 30_1\lyxdot 00_3\lyxdot 00}\includegraphics[width=0.5\textwidth]{tmacc_rg_plot/tmacc_rnd_a_0\lyxdot 30_1\lyxdot 00_3\lyxdot 00}
\par\end{centering}

\begin{centering}
\includegraphics[width=0.5\textwidth]{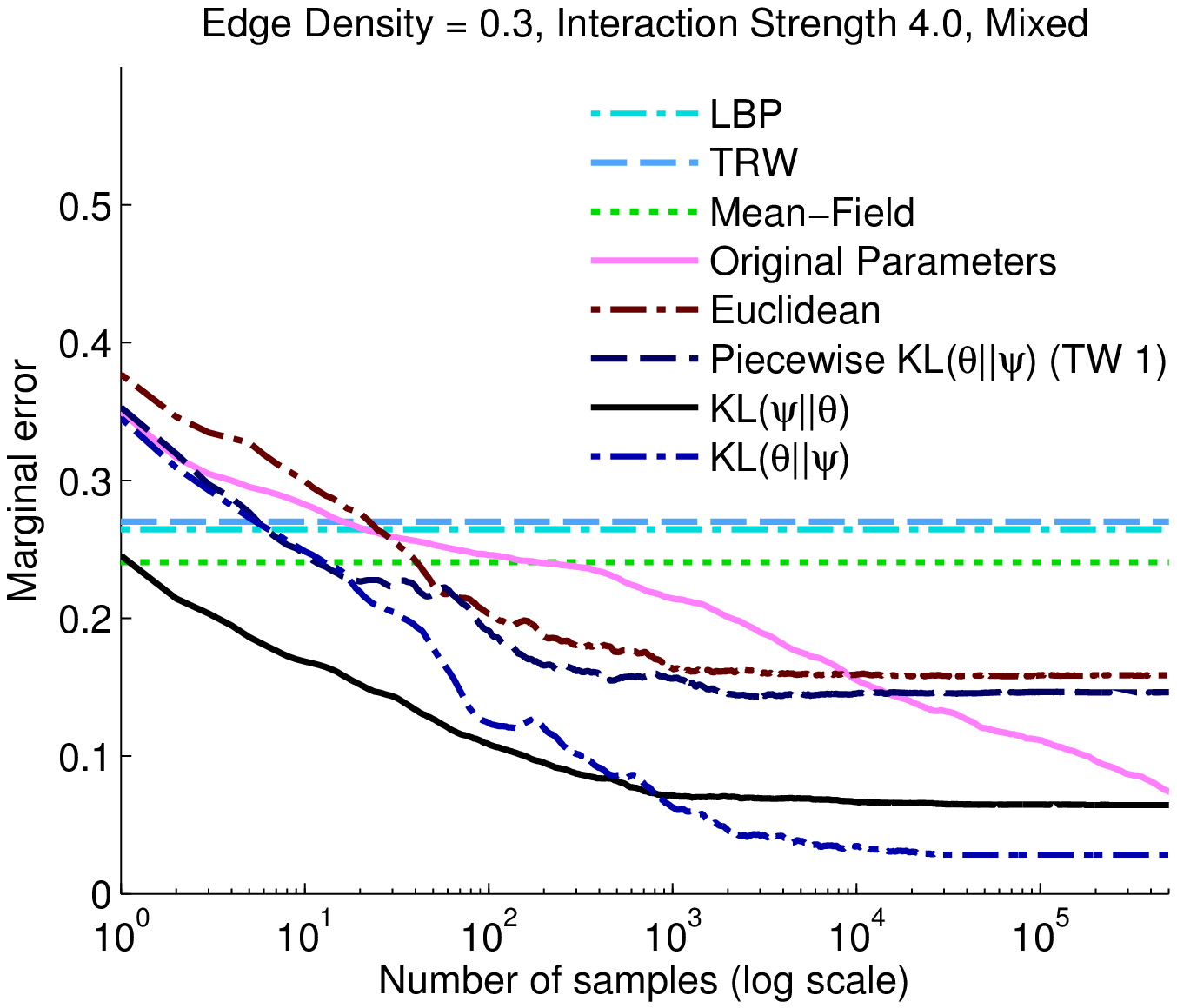}\includegraphics[width=0.5\textwidth]{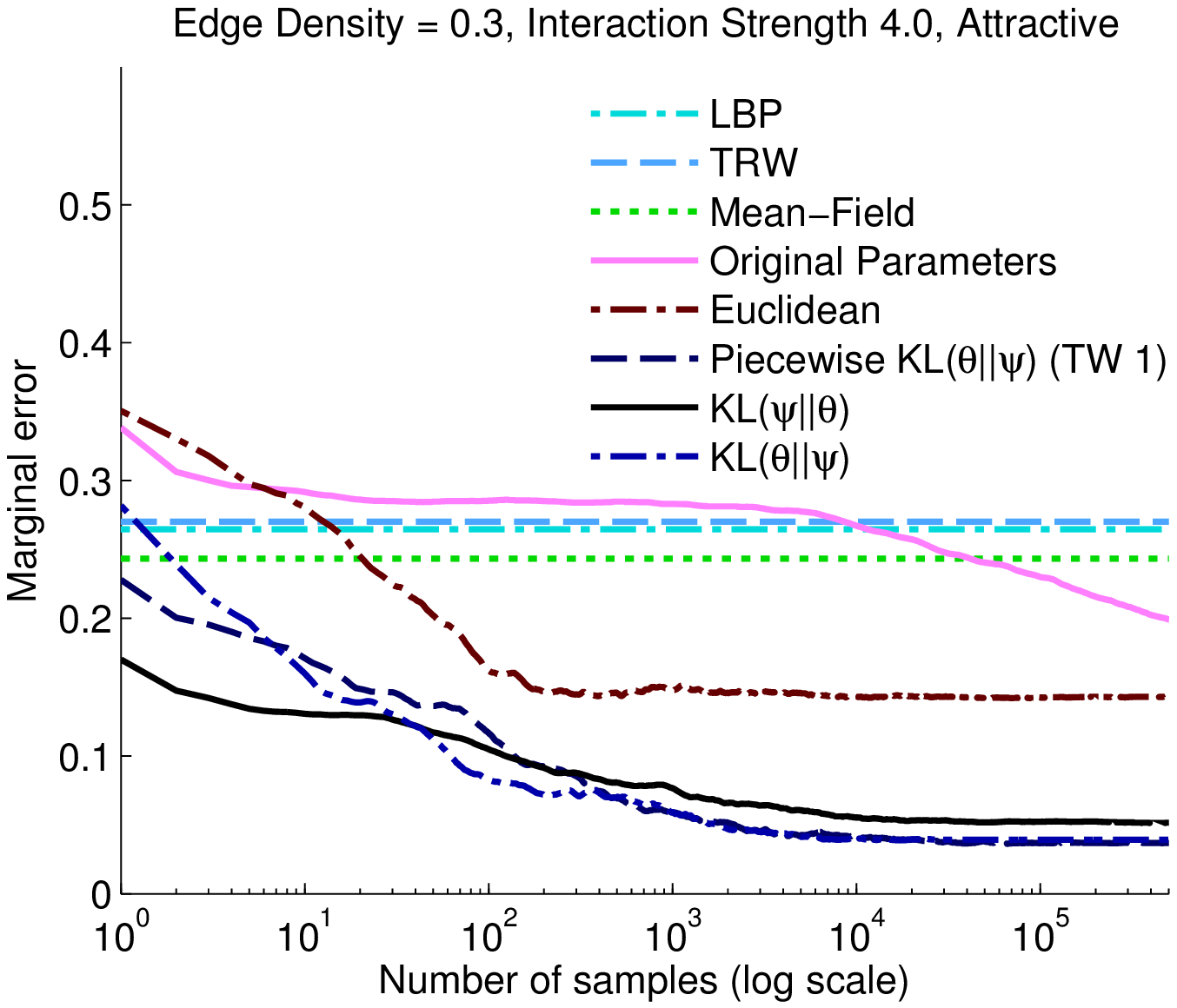}
\par\end{centering}

\caption{Accuracy on Low-Density Graphs as a function of time}
\end{figure}

\begin{figure}
\begin{centering}
\includegraphics[width=0.5\textwidth]{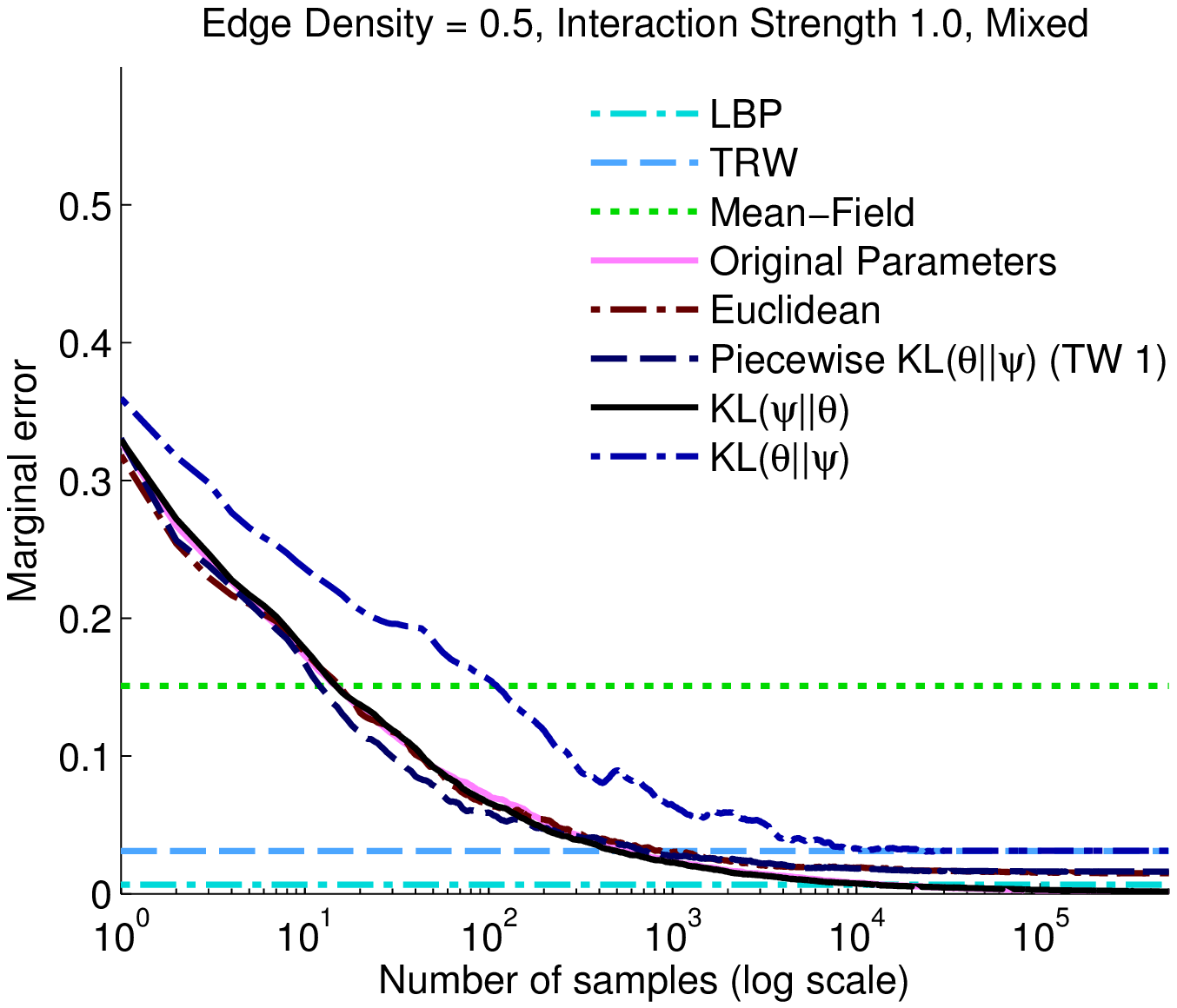}\includegraphics[width=0.5\textwidth]{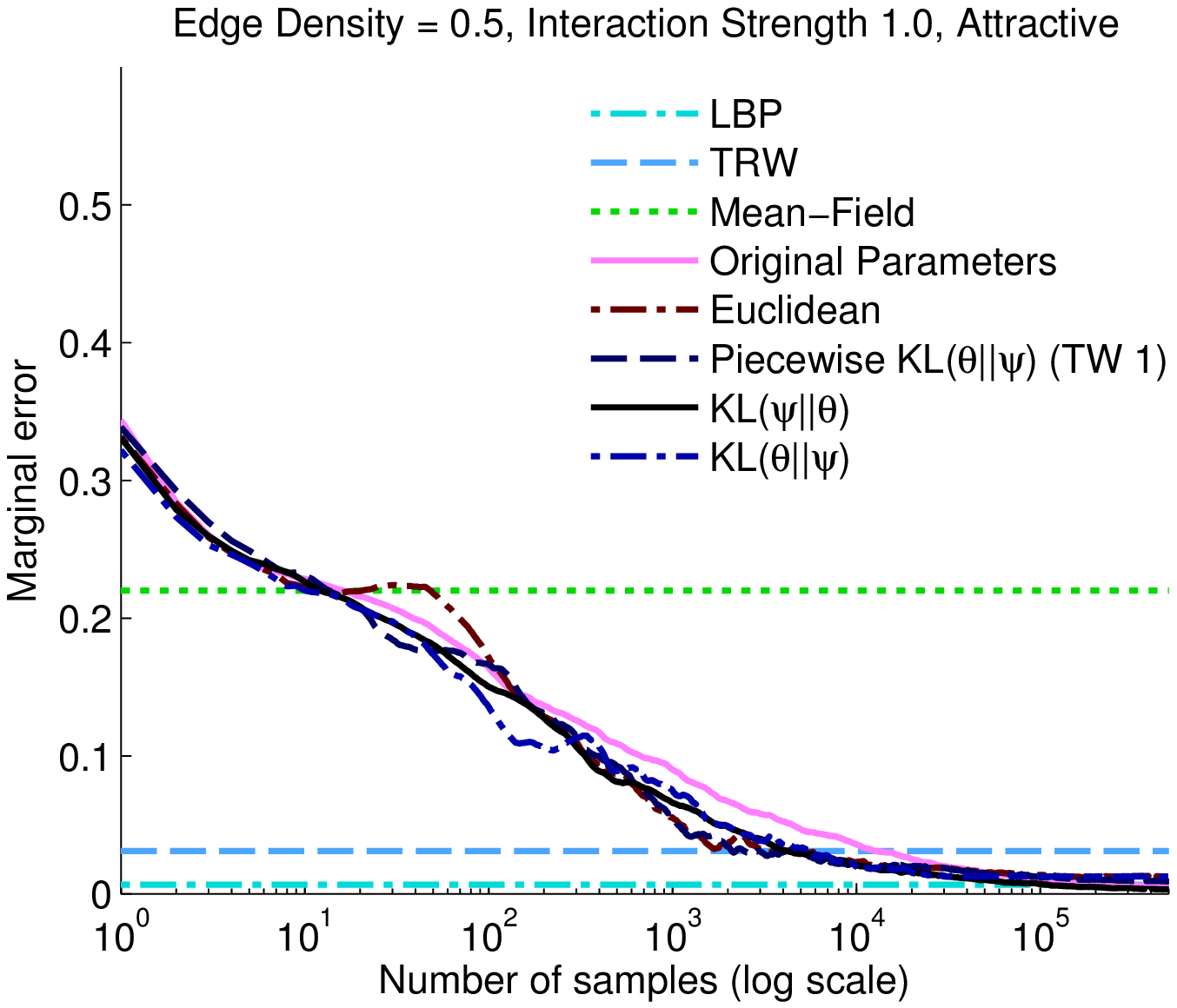}
\par\end{centering}

\begin{centering}
\includegraphics[width=0.5\textwidth]{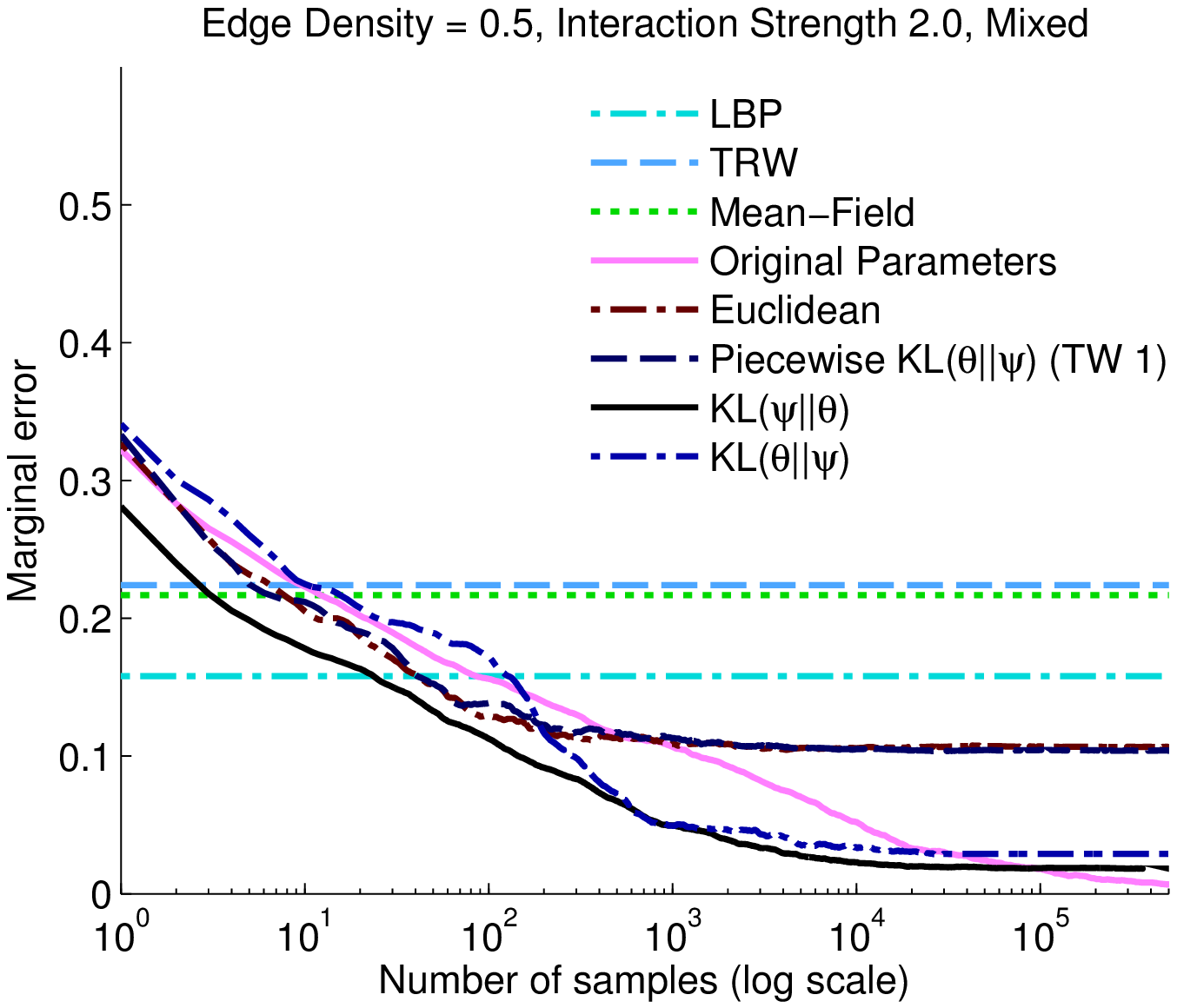}\includegraphics[width=0.5\textwidth]{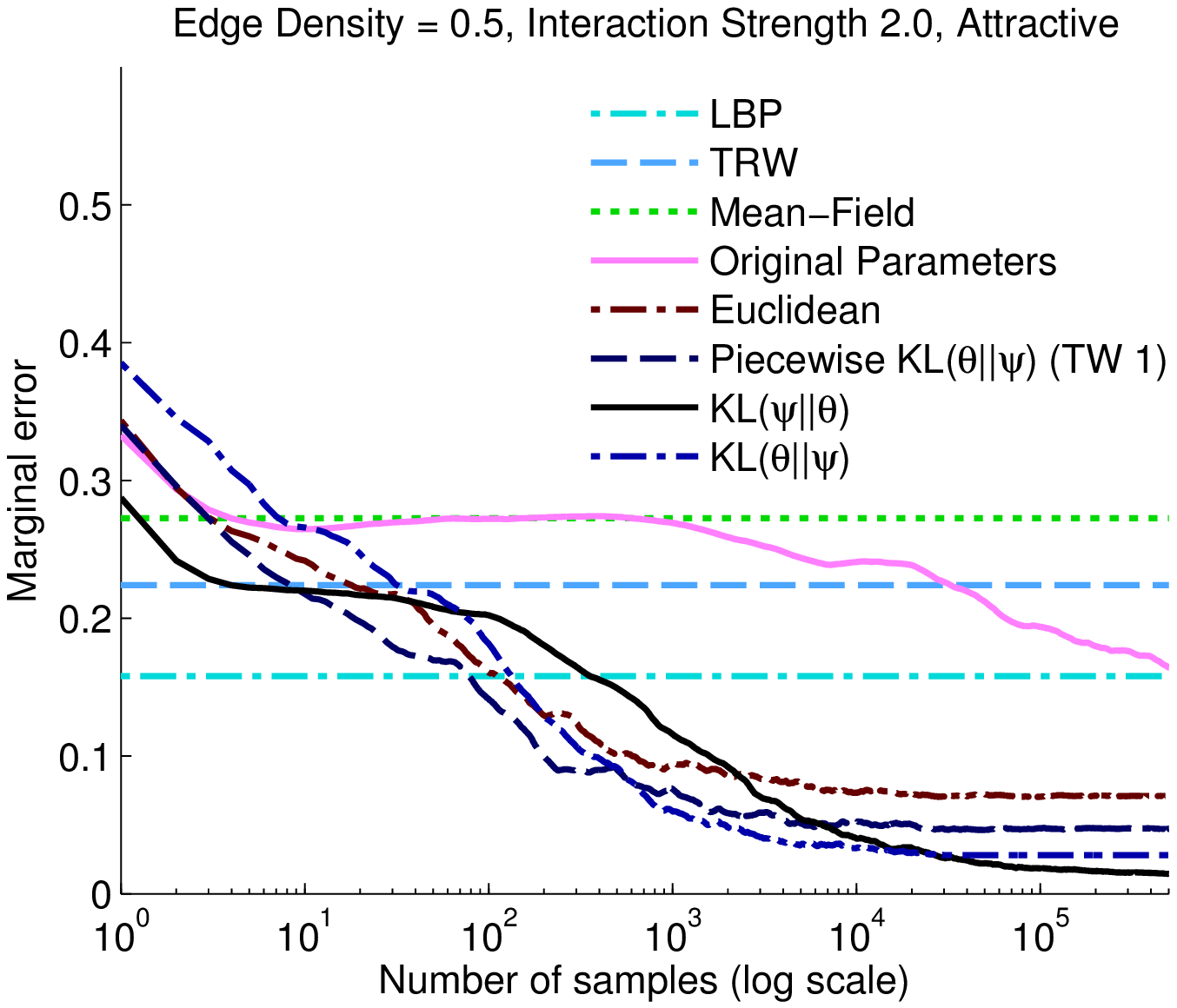}
\par\end{centering}

\begin{centering}
\includegraphics[width=0.5\textwidth]{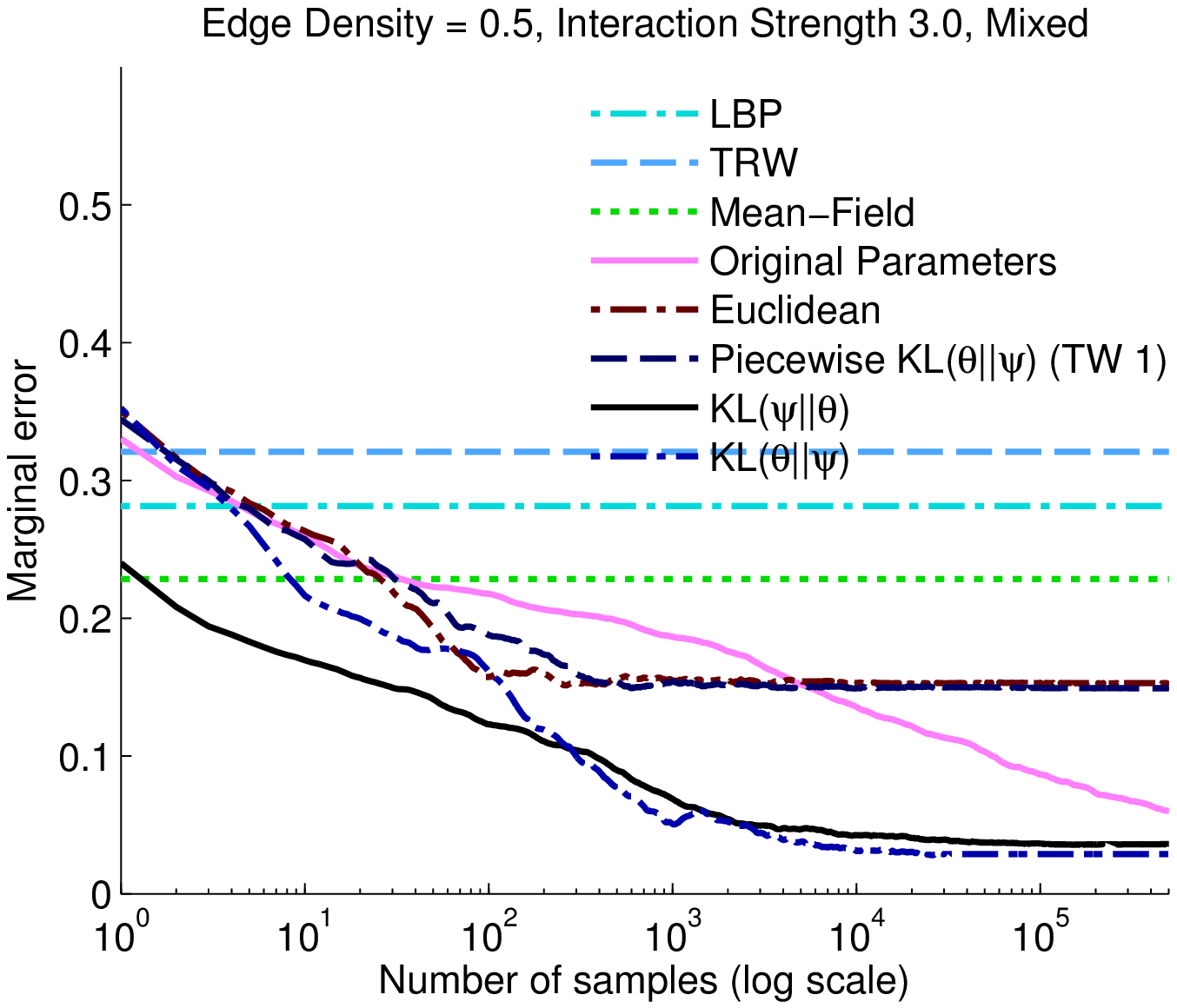}\includegraphics[width=0.5\textwidth]{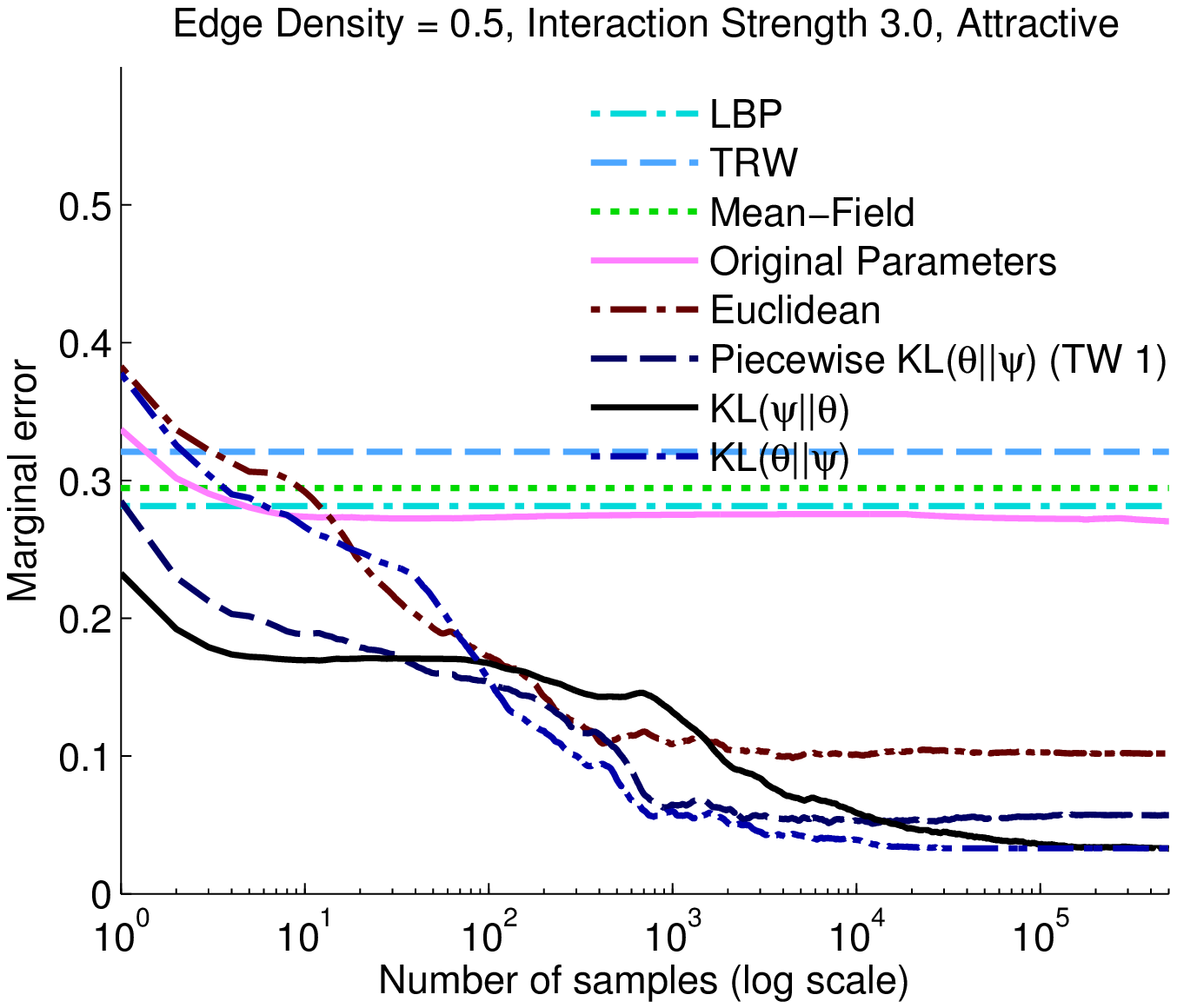}
\par\end{centering}

\begin{centering}
\includegraphics[width=0.5\textwidth]{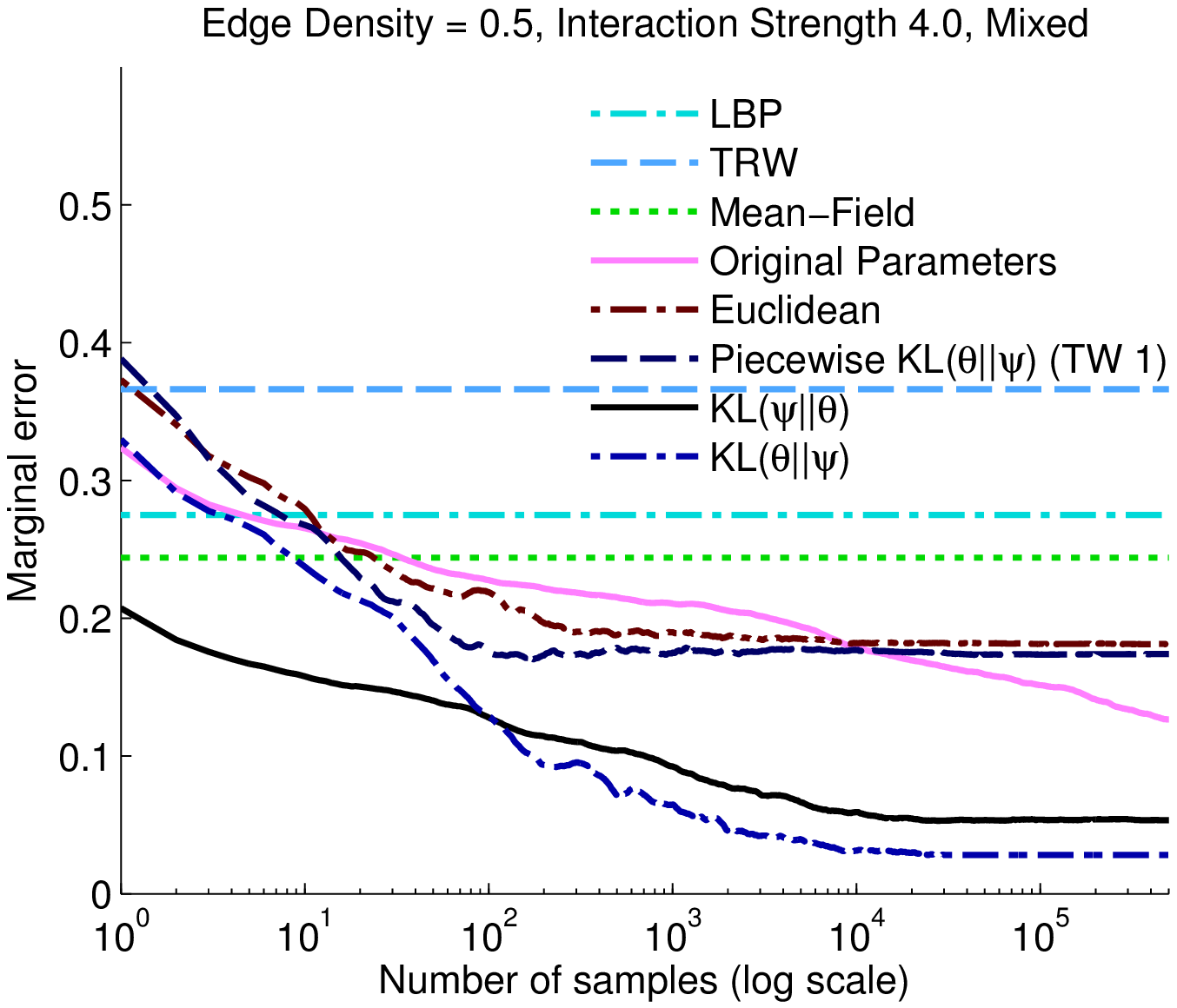}\includegraphics[width=0.5\textwidth]{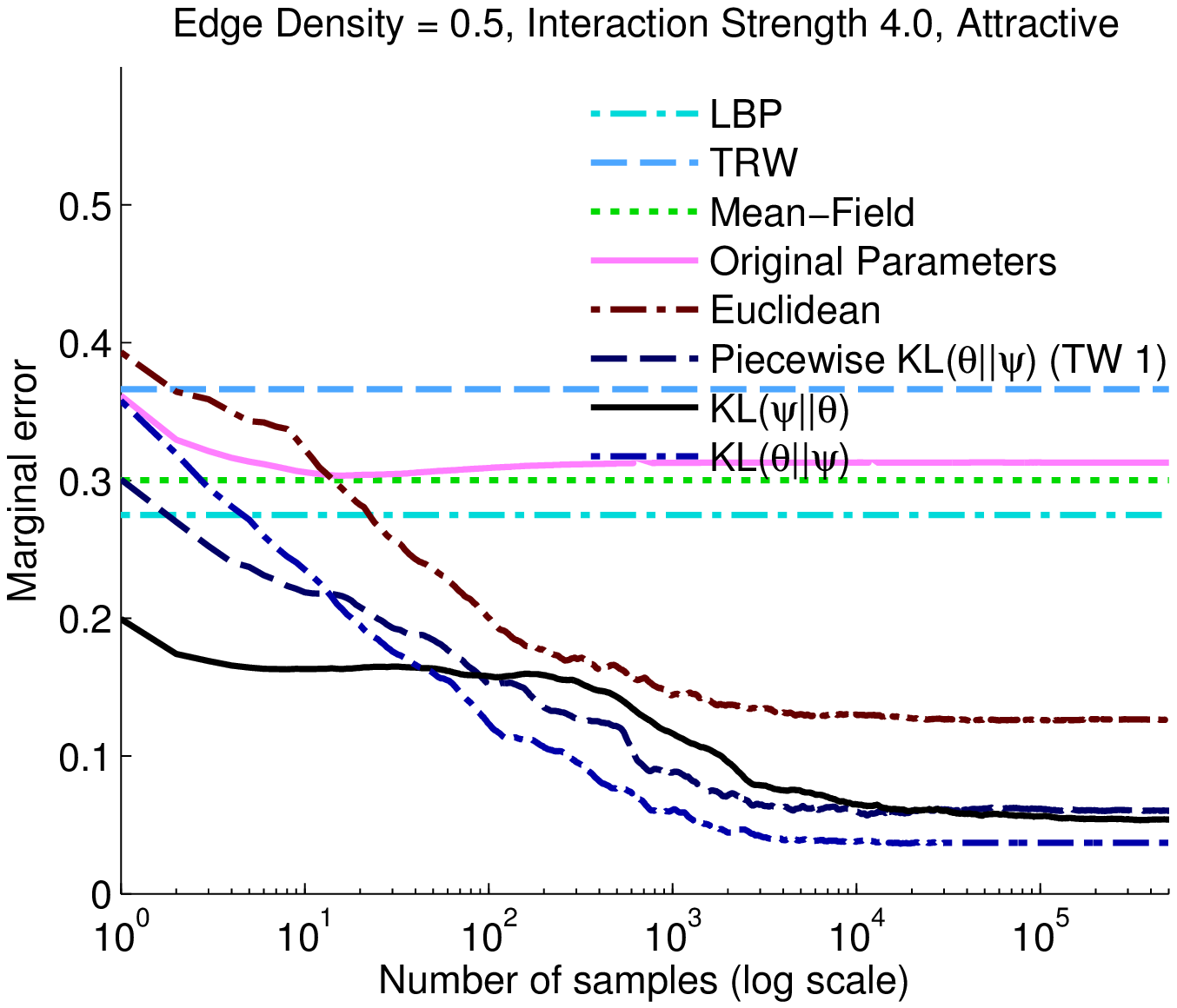}
\par\end{centering}

\caption{Accuracy on Medium-Density Graphs as a function of time}
\end{figure}

\begin{figure}
\begin{centering}
\includegraphics[width=0.5\textwidth]{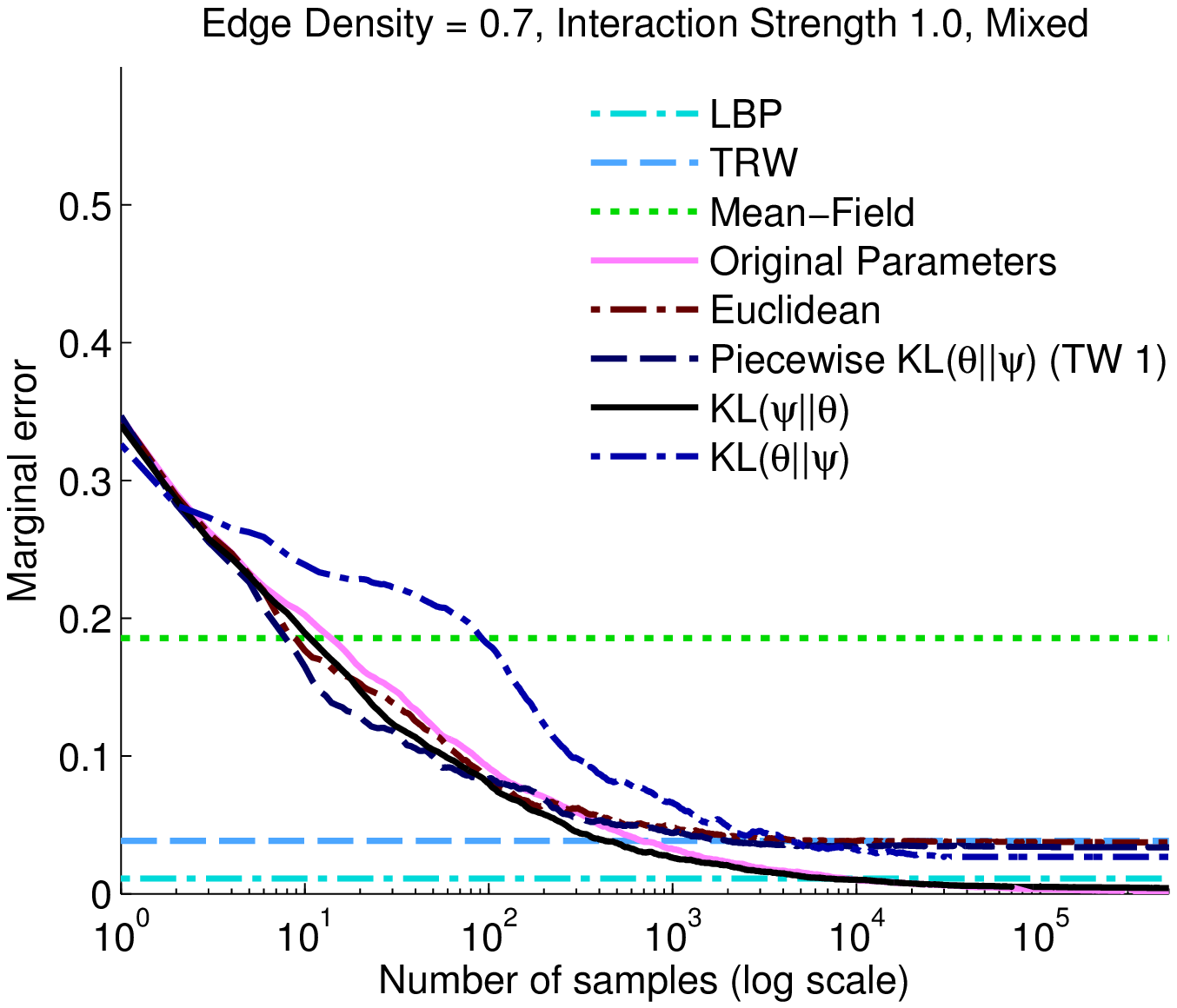}\includegraphics[width=0.5\textwidth]{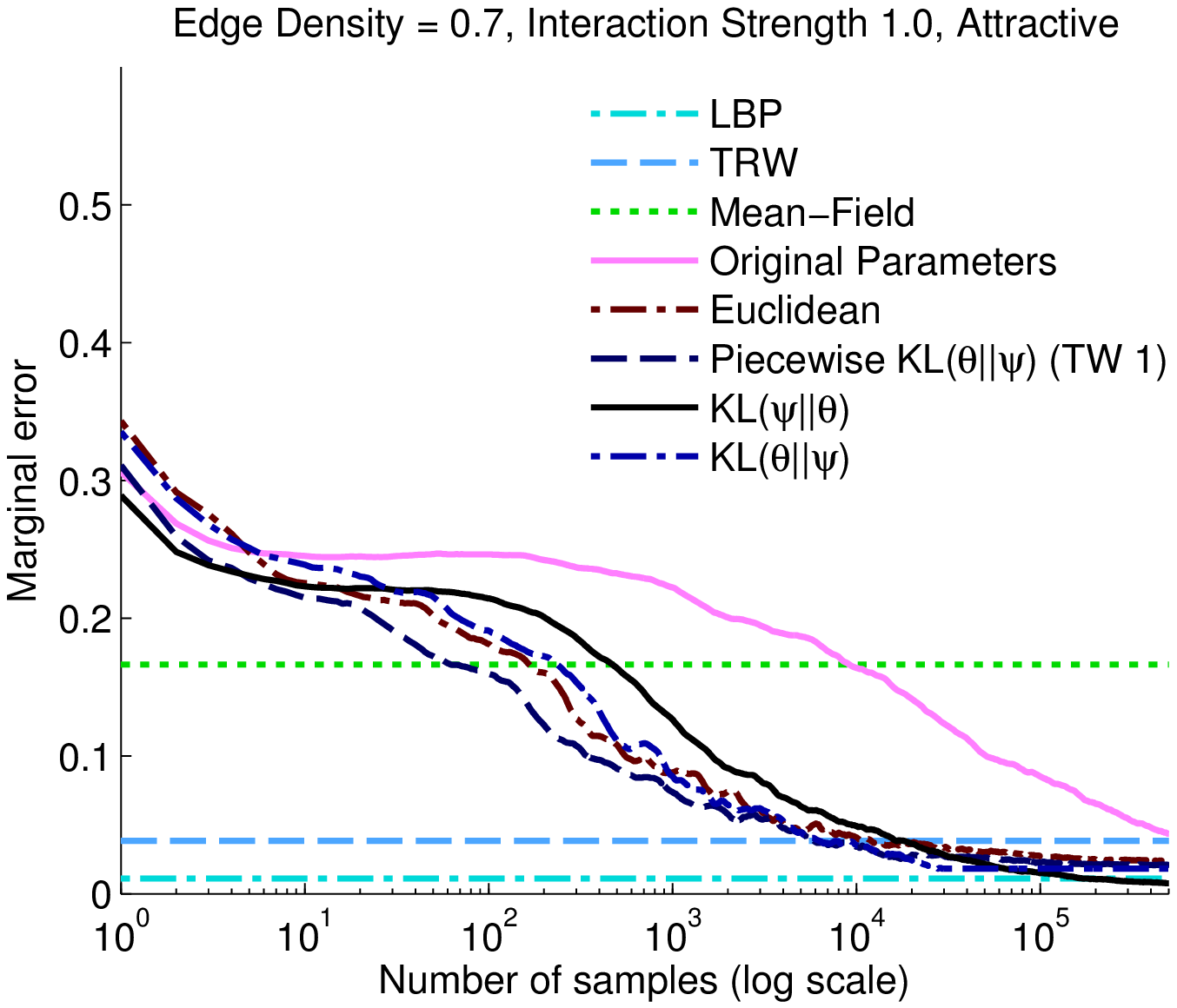}
\par\end{centering}

\begin{centering}
\includegraphics[width=0.5\textwidth]{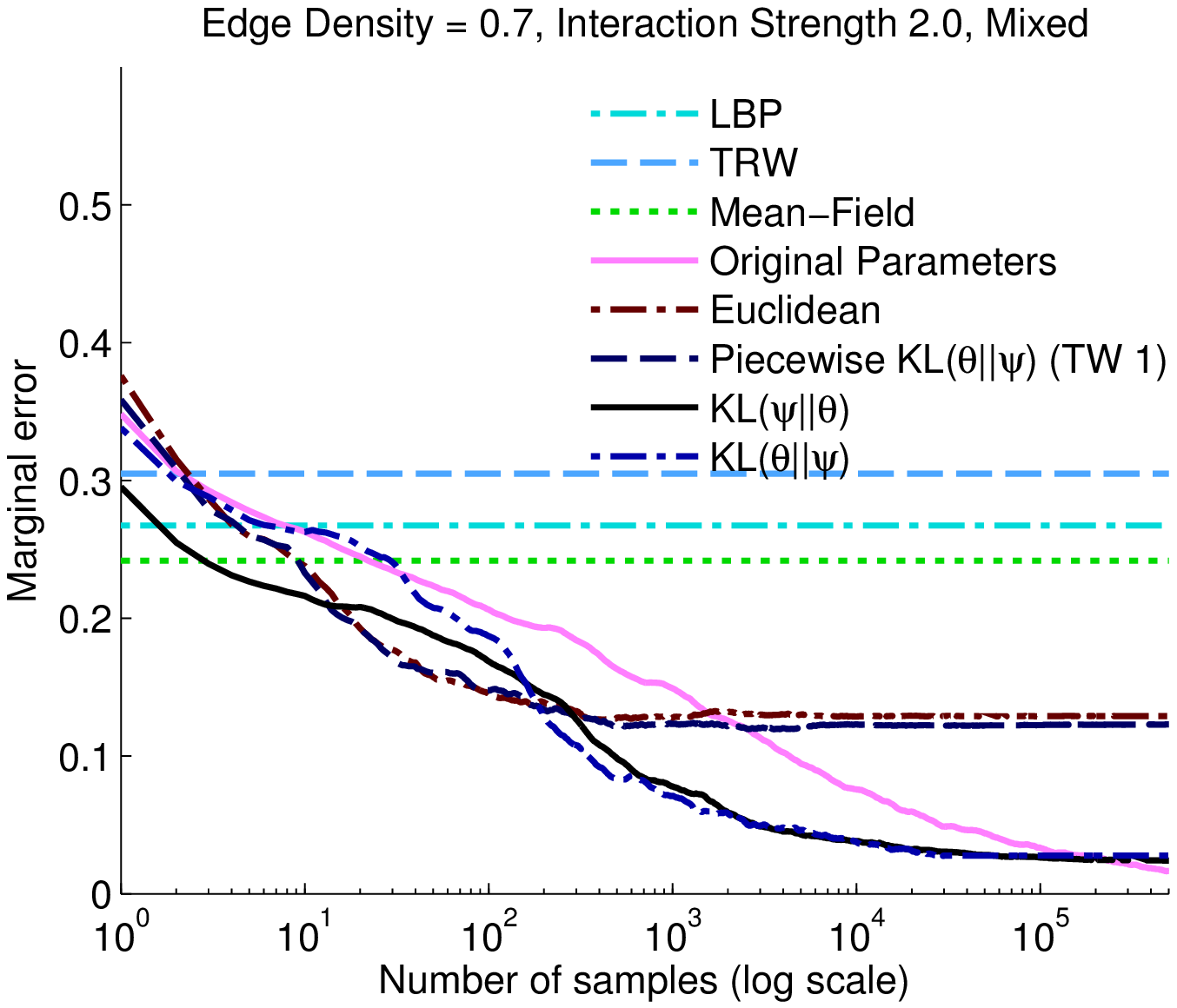}\includegraphics[width=0.5\textwidth]{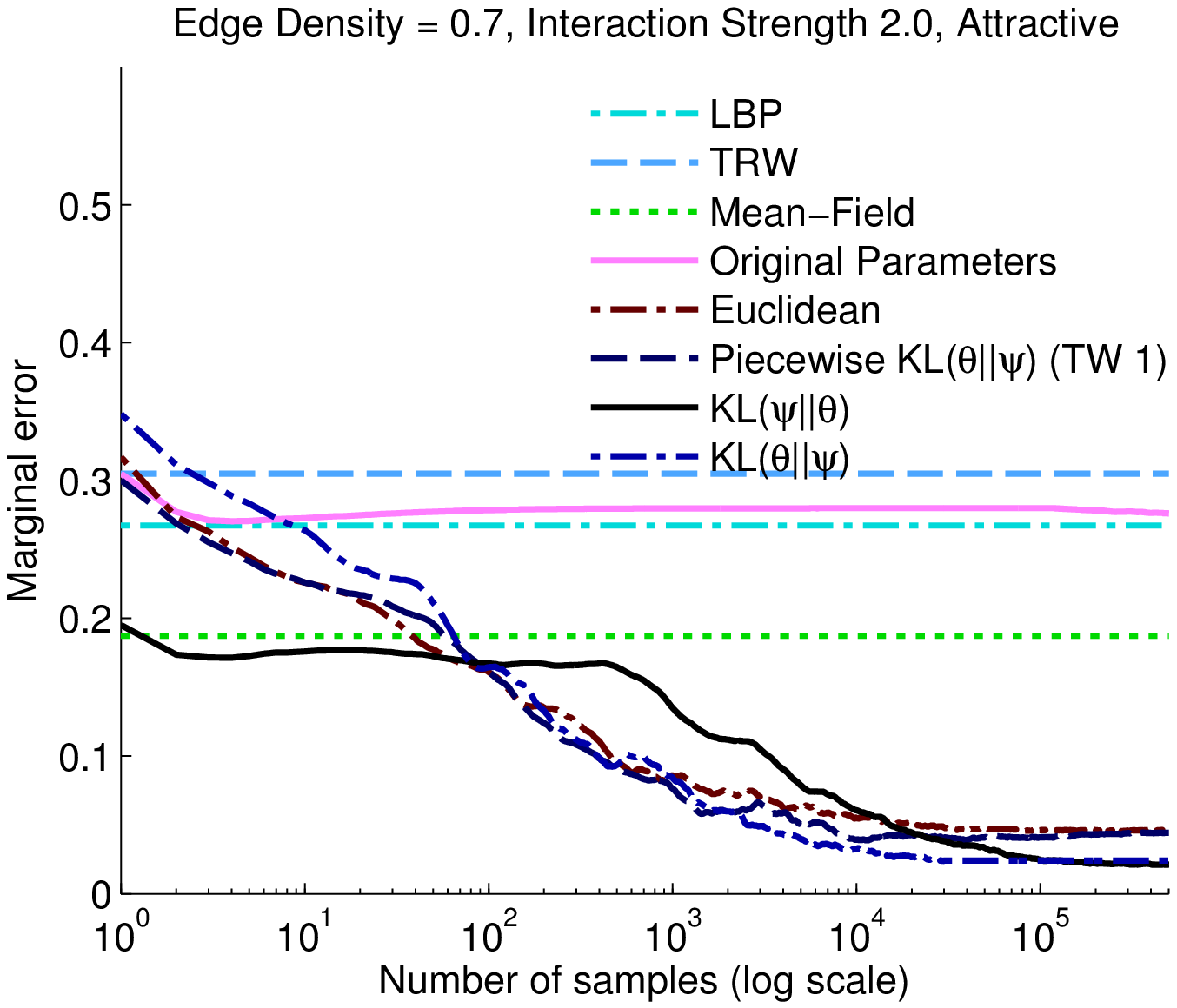}
\par\end{centering}

\begin{centering}
\includegraphics[width=0.5\textwidth]{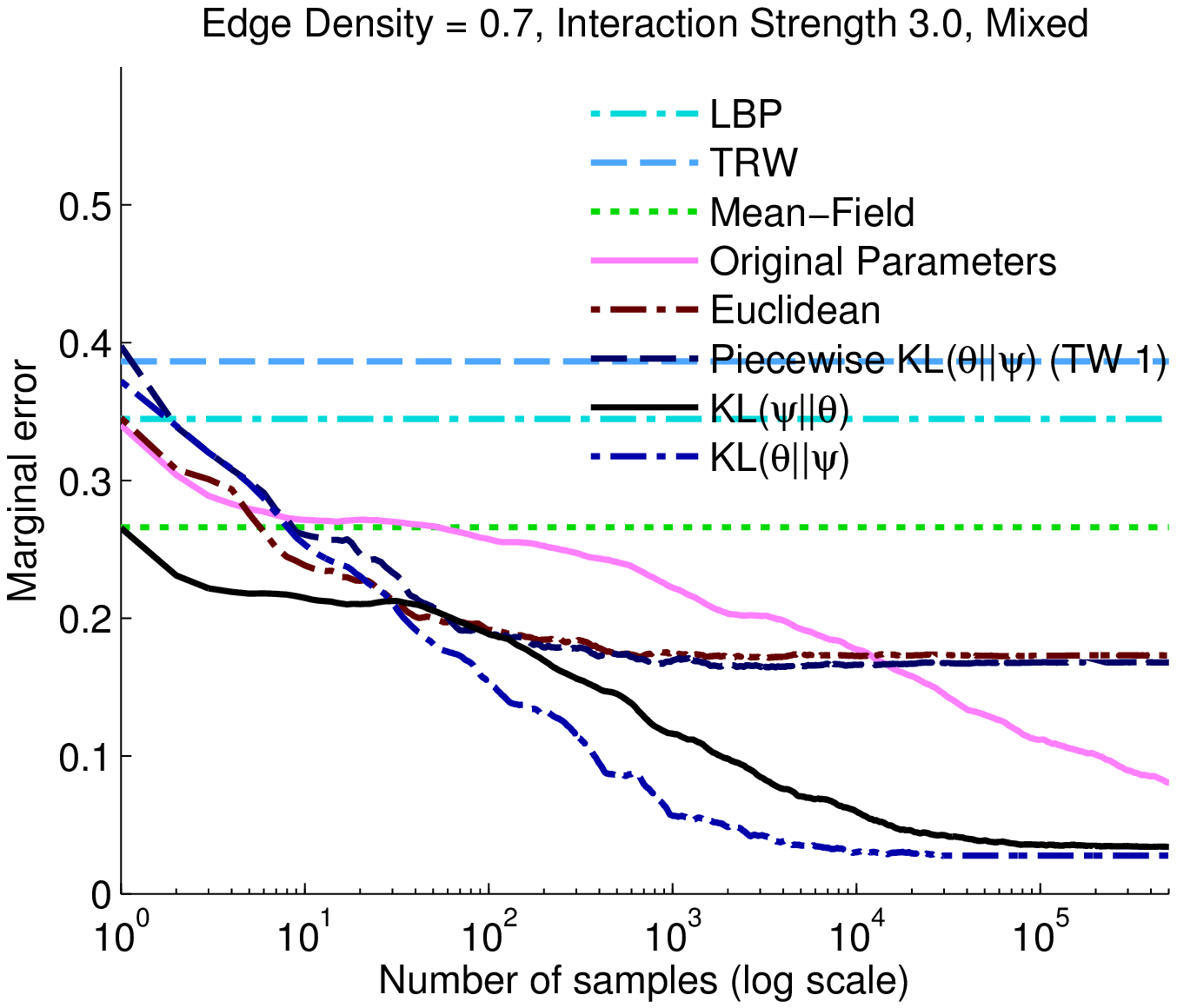}\includegraphics[width=0.5\textwidth]{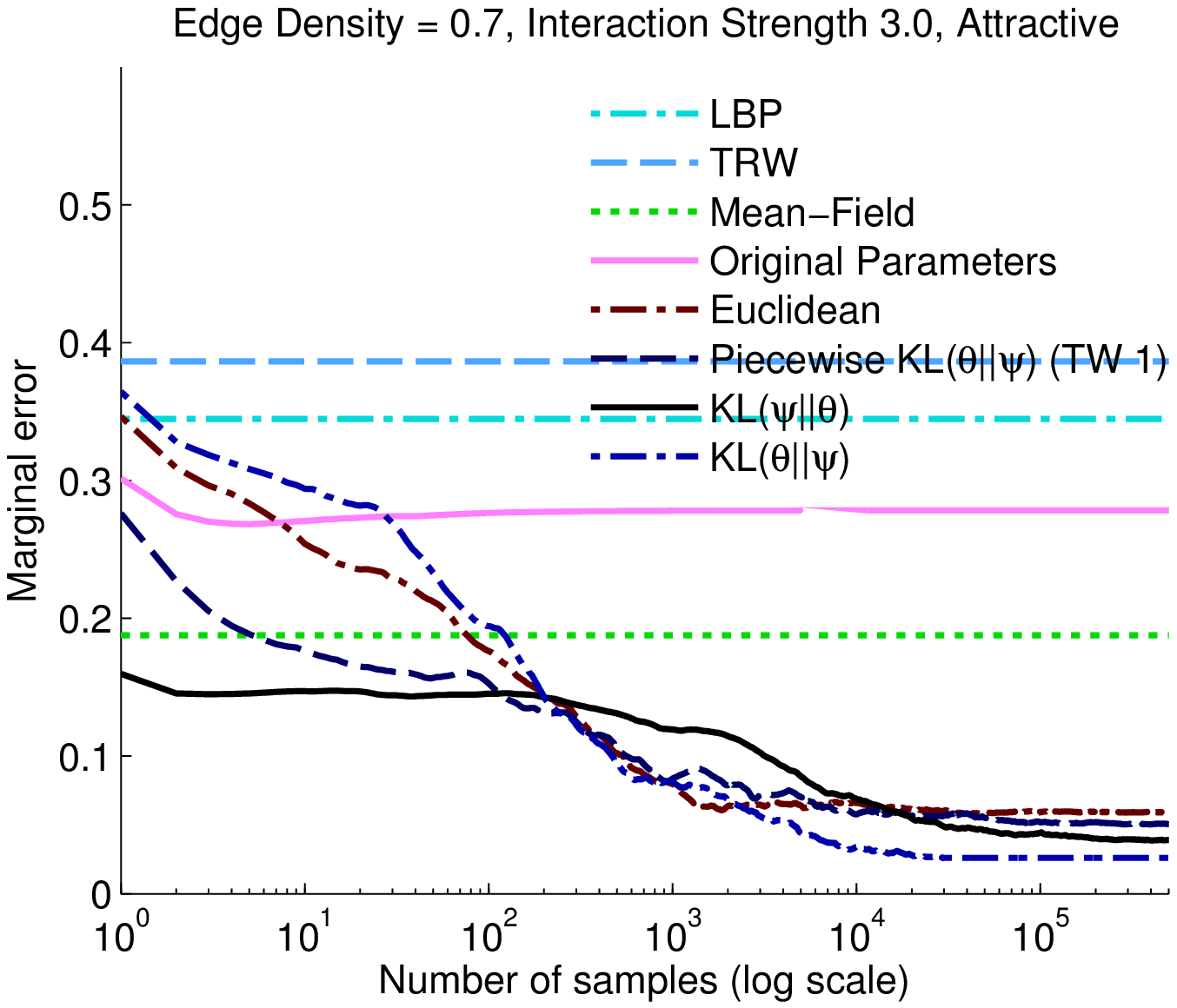}
\par\end{centering}

\begin{centering}
\includegraphics[width=0.5\textwidth]{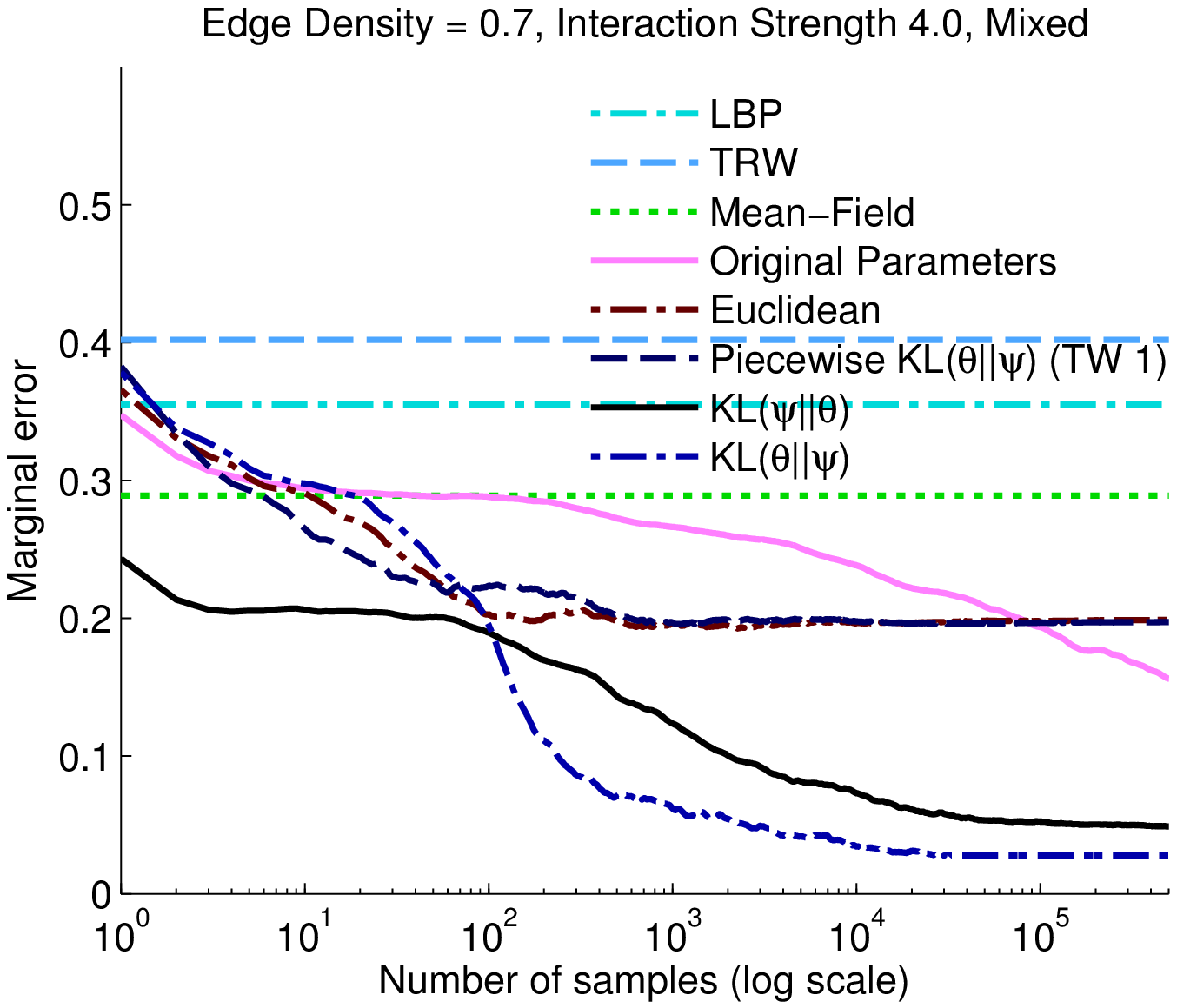}\includegraphics[width=0.5\textwidth]{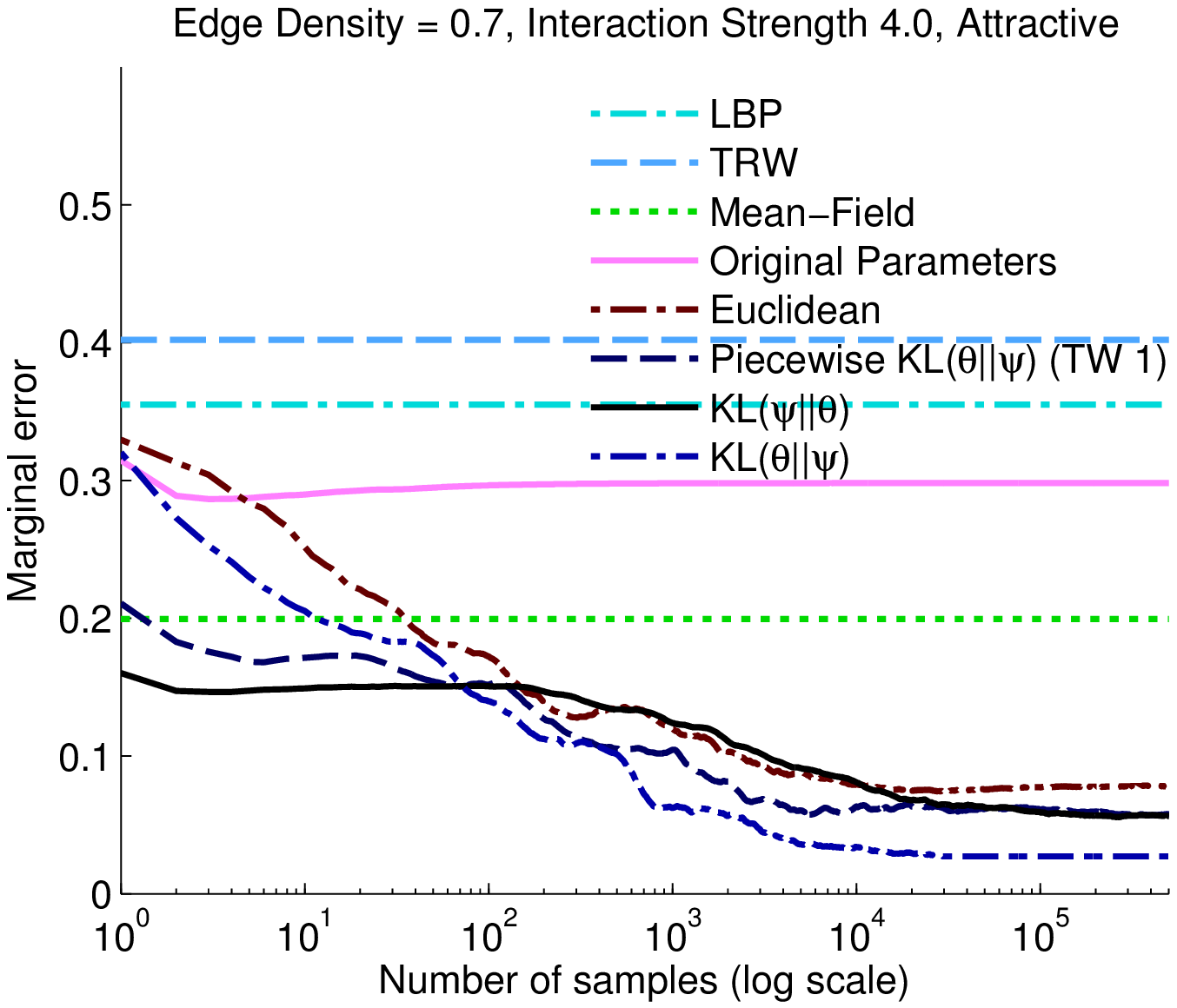}
\par\end{centering}

\caption{Accuracy on High-Density Random Graphs as a function of time}
\end{figure}

\end{document}